\documentclass[final]{IEEEtran}

\usepackage{amssymb,url,amsmath, amsfonts,relsize, accents}
\usepackage{enumerate,color}
\usepackage{graphicx}
\usepackage{multicol}
\usepackage{galois}
\usepackage{amsthm} 
\usepackage{authblk}
\usepackage{subcaption}

\usepackage{amssymb}
\usepackage{amsmath,amsthm}
\DeclareMathOperator{\tr}{tr}
\usepackage{subcaption}
\usepackage{amsthm} 
\newtheorem{prop}{Proposition}

\usepackage{amsthm} \newtheorem{Remark}{Remark}

\usepackage{amsthm} 
\usepackage{amsthm} 
\usepackage{amsthm} 
\usepackage{amsthm} 
\usepackage{enumitem}
\usepackage{algorithm}
\usepackage{algorithmic}
\newtheorem{theorem}{Theorem}

\newtheorem{lemma}{Lemma}

\newtheorem{assumption}{Assumption}
 \newcounter{experiment}[section]

\usepackage{graphicx} % more modern

\usepackage{times}
\usepackage{graphicx} % more modern

\usepackage{hyperref}

\begin{document}

\title{Learning Latent Features with Pairwise Penalties \\in Low-Rank Matrix Completion}

\author{Kaiyi~Ji, Jian~Tan, Jinfeng~Xu and Yuejie~Chi,~\IEEEmembership{Senior Member,~IEEE} 
\thanks{K. Ji and J. Tan are with Department of Electrical and Computer Engineering, The Ohio State University; e-mail: {\tt\{ji.367, tan.252\}@osu.edu}.
This work is supported in part by the National Science Foundation under Grant No. 1717060.}%OH 43210, USA
%PA 15213, USA
\thanks{J. Xu is with Department of Statistics and Actuarial Science, The University of Hong Kong; e-mail: {\tt xujf@hku.hk}.}
\thanks{Y. Chi is with Department of Electrical and Computer Engineering, Carnegie Mellon University; e-mail: {\tt yuejiechi@cmu.edu}. The work of Y. Chi is supported in part by ONR under the grant N00014-18-1-2142, by ARO under the grant ARO W911NF-18-1-0303, and by NSF under CIF-1806154, CIF-1826519 and ECCS-1818571.}
}

\maketitle

\begin{abstract} 

Low-rank matrix completion has achieved great success in many real-world data applications.
A matrix factorization model that learns latent features is usually employed and, to improve prediction performance, the 
similarities between latent variables can be exploited by pairwise learning using the graph regularized 
matrix factorization (GRMF) method. However, existing GRMF approaches often use the squared loss to measure the pairwise differences, which may be overly influenced by dissimilar pairs and lead 
to inferior prediction. To fully empower pairwise learning for matrix completion, we propose a general 
optimization framework that allows a rich class of (non-)convex pairwise penalty functions. A new 
and efficient algorithm is developed to solve the proposed optimization problem, with a theoretical 
convergence guarantee under mild assumptions. In an important situation where the latent variables form a small 
number of subgroups, its statistical guarantee is also fully considered. In particular, we theoretically 
characterize the performance of the complexity-regularized maximum likelihood estimator, as a special case of our framework, which is shown to have smaller errors when compared to the standard matrix completion framework without pairwise penalties. We conduct extensive experiments on both synthetic and real datasets to demonstrate the superior
performance of this general framework.

\end{abstract} 


\section{Introduction}
 
Low-rank matrix completion \cite{chi2018nonconvex,chen2018harnessing} has been widely used in many real-world data applications, e.g., image restoration and collaborative filtering. 
A typical optimization~problem follows the form~\cite{Koren:2009,gopalan2014content} 
\begin{align}\label{eq:genform1}
 \min_{\mathbf{X}, \mathbf{Y}}  \; \frac{1}{2}\left\|\Psi_{\Omega}(\mathbf{M}-\mathbf{X}\mathbf{Y}^T)\right\|_F^2+ \frac{\alpha}{2}\left( \|\mathbf{X}\|^2_F+\|\mathbf{Y}\|^2_F\right),
\end{align}
where $\mathbf{M}\in\mathbb{R}^{n\times m}$ is the data matrix, $\mathbf{X}\in\mathbb{R}^{n\times d}$, $\mathbf{Y}\in\mathbb{R}^{m\times d}$ are the low-rank factors, with typically $d\ll \min\{n,m\}$. The projection operator
$\Psi_{\Omega}(\mathbf{M})$ retains the entries of the matrix $\mathbf{M}$ in the index set $\Omega$ that denotes the observed indices, and $\alpha>0$ is some regularization parameter.

The row vectors $\{\mathbf{x}_i\in\mathbb{R}^{d}\}_{i=1}^n$ of $\mathbf{X}=[\mathbf{x}_1,\cdots, \mathbf{x}_n]^T$ and $\{\mathbf{y}_j\in\mathbb{R}^d\}_{j=1}^m$ of $\mathbf{Y}=[\mathbf{y}_1,\cdots, \mathbf{y}_m]^T$, known as {\em latent variables}, usually represent the features of two classes of interdependent objects, e.g., user features and movie features in recommender systems~\cite{Koren:2009,Bennett07thenetflix}, respectively. 
Beyond this basic model, many variants have been considered for different application scenarios to further improve the estimation accuracy. 
For example, the similarities between latent variables can be exploited by pairwise learning. Define a graph $\mathcal{G}_X=({V}_X,{E}_X,\mathbf{W})$, where the feature vectors $\{\mathbf{x}_i\}_{i=1}^n$ correspond to the vertices $V_X = \{ 1, 2, \cdots, n \} $ of the graph, and the edges $E_X = V_X \times V_X$ are weighted by entries in a non-negative matrix $\mathbf{W}=[w_{ij}]\in\mathbb{R}^{n\times n}$. Similarly, define the graph $\mathcal{G}_Y = (V_Y,E_Y,\mathbf{U})$ for the feature vectors $\{\mathbf{y}_j\}_{j=1}^m$. The graph regularized matrix factorization (GRMF) method \cite{rao2015collaborative,zhao2015expert,monti2017geometric} aims to solve the following problem:
\begin{align}\label{eq:genform3}
 \min_{\mathbf{X}, \mathbf{Y}} \;  \frac{1}{2}\left\|\Psi_{\Omega}(\mathbf{M}-\mathbf{X}\mathbf{Y}^T)\right\|_F^2+ \frac{\alpha}{2}\left( \|\mathbf{X}\|^2_F+\|\mathbf{Y}\|^2_F\right)\nonumber
 \\+\gamma_X \sum\limits_{i<j} w_{ij} \|\mathbf{x}_{i} -\mathbf{x}_{j}\|_2^2+\gamma_Y \sum\limits_{s<t} u_{st}\|\mathbf{y}_{s} -\mathbf{y}_{t}\|_2^2,
\end{align}
which adds a smoothing graph regularizer to~(\ref{eq:genform1}), a technique also known as  Laplacian smoothing \cite{smola2003kernels}. This encourages the feature vectors $\mathbf{x}_{i} \approx  \mathbf{x}_{j}$ (respectively $\mathbf{y}_{s} \approx \mathbf{y}_{t}$) if the weight $w_{ij}$ (respectively $u_{st}$) is large. It has been demonstrated that GRMF has the potential to reduce the estimation error and improve prediction performance~\cite{rao2015collaborative,zhao2015expert,monti2017geometric}. 

\subsection{Our contributions}

It is commonly known that Laplacian smoothing using the squared loss as in \eqref{eq:genform3} tends to enforce smoothness property in a global manner, and does not adapt well to inhomogeneity across different nodes \cite{wang2015trend}. Indeed, as illustrated in Fig.~\ref{fig:concave}, the squared loss in \eqref{eq:genform3} tends to impose a large penalty when $\mathbf{x}_i$ is highly different from $\mathbf{x}_j$; correspondingly, if $w_{ij}$ is not so small, GRMF tends to push the solution $\mathbf{x}_{i} $ to be close to $\mathbf{x}_{j}$ which can lead to biased estimates as shown in~\cite{ma2017concave}, and dramatically affect the recovery performance. This is exacerbated further when the weight matrices $\mathbf{W}$ and $\mathbf{U}$ are constructed using side information of the feature vectors, e.g., the friendships of users and the attributes of movies in recommender systems, which could be very noisy \cite{chiang2015matrix} and inappropriately selected.

\begin{figure*}[t]
\centering
\includegraphics[height=1.5in,width=0.65\textwidth]{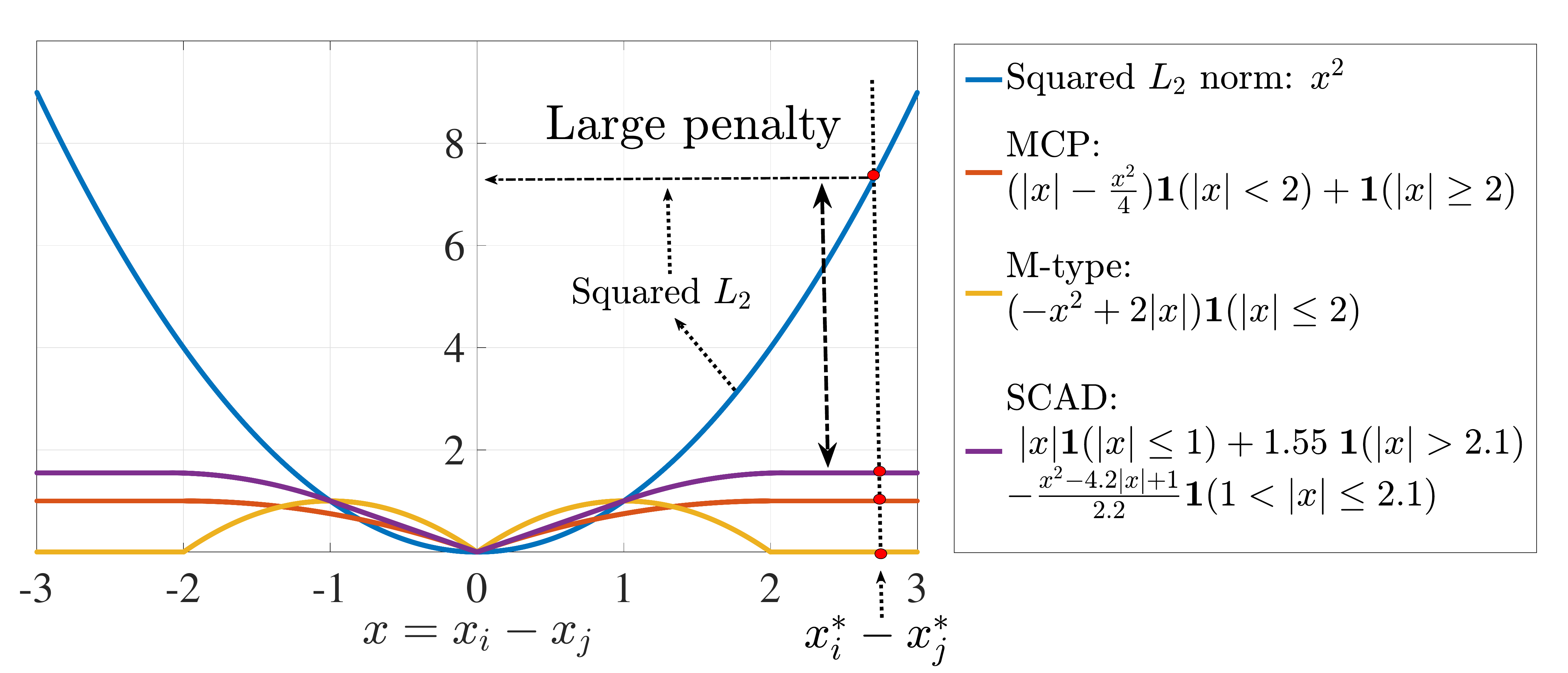} 
\caption{Illustration of representative pairwise penalty functions $p(x, \gamma)$ \text{w.r.t.} a scalar $x $.  }\label{fig:concave}
\vspace{-0.3cm}
\end{figure*}

In this work, we introduce a large family of (non-)convex pairwise penalty functions in the low-rank matrix factorization framework, by proposing a general optimization framework
\begin{align}\label{eq:genform4}
  &\min_{\mathbf{X}, \mathbf{Y}} \; \frac{1}{2}\left\|\Psi_{\Omega}(\mathbf{M}-\mathbf{X}\mathbf{Y}^T)\right\|_F^2+ \frac{\alpha}{2}\left( \|\mathbf{X}\|^2_F+\|\mathbf{Y}\|^2_F\right) 
  \\ &+\sum\limits_{i<j} w_{ij} p\left(\|\mathbf{x}_{i} -\mathbf{x}_{j}\|_2,\gamma_{X}\right)+ \sum\limits_{s<t}u_{st}  p\left(\|\mathbf{y}_{s}-\mathbf{y}_{t}\|_2,\gamma_{Y} \right),\nonumber
\end{align}
where
$p(\cdot, \gamma)$ is a (non-)convex pairwise penalty function
with a tuning parameter $\gamma\geq 0$. The pairwise penalty function takes Laplacian smoothing as a special case, where $p(x,\gamma)=\gamma x^2$. More importantly, it offers the opportunity to incorporate other important regularizers such as $\ell_1$ penalty by setting $p(x,\gamma) = \gamma |x|$, which has proven its value in a different context for graph trend filtering \cite{wang2015trend,Varma2019vector}. Furthermore, non-convex pairwise penalties can be used, such as the minimax concave penalty (MCP)~\cite{zhang2010nearly}, the smoothly clipped absolute deviation (SCAD) penalty \cite{fan2001variable} and the M-type function, as illustrated in Fig.~\ref{fig:concave} (the details will be given later). To the best of our knowledge, the power of general families of pairwise penalties in the context of low-rank matrix factorization has not been studied before and is the focus of this work.

With the optimization problem \eqref{eq:genform4} in mind, our contributions in this paper are summarized below. First, we design 
a \textit{novel and scalable} algorithm based on 
a modified alternating direction method of multipliers (ADMM)~\cite{Boyd:2011} to efficiently solve the proposed problem \eqref{eq:genform4}, by introducing Bregman divergence to certain subproblems.
The allowed non-convex and non-smooth pairwise penalty functions complicate the optimization, which, if not handled carefully, can even result in divergent iterations. 
Theoretically, we
characterize the convergence of our algorithm for a large family of 
pairwise penalty functions satisfying some mild properties. This requires new analysis since recent works on the convergence of ADMM for non-convex functions~\cite{wang2018convergence,hong2016convergence,wang2015global} cannot be directly applied to our problem. Moreover, 
the modified ADMM algorithm needs to solve an expensive linear system of equations that typically requires  
$O((m^2+n^2)\times d^2)$ run time
to invert two large matrices in each iteration.
We provide a conjugate gradient (CG) based approach to obtain an \textit{inexact} solution to this linear system of equations, using only $O(|\Omega|\times d)$ run time.
Notably, we still guarantee the convergence using the inexact solutions during iterations.

Second, we study the statistical benefits of incorporating pairwise penalties. Fully characterizing the statistical guarantees for the proposed class of penalty functions is challenging. Instead, we restrict ourselves to an important
situation where the latent variables form a small number of {\em subgroups}. 
Specifically, we investigate a subgroup-based model, where the feature vectors are considered in the same group if they are identical. For a partially observed matrix corrupted with Gaussian noise, we prove that the complexity-regularized maximum likelihood estimator, as a special case of the framework~(\ref{eq:genform4}),
can achieve a smaller error bound than the one without considering the pairwise penalties,
especially when the numbers of the subgroups are small. The class of sparsity-promoting pairwise penalty functions (possibly with a finite support) can be regarded as a computationally-efficient surrogate of the complexity-regularized maximum likelihood estimator. Interestingly, not only can we identify the subgroups {automatically} but also we
 significantly reduce the recovery error, as verified in Section~\ref{reducedError}.
  
 We propose a few heuristic rules to adaptively construct the weight matrices of the graph based on the partially observed matrix. As a result,  our framework also applies to datasets that do not provide any side information. 
Our extensive experiments on both synthetic and real data  demonstrate the superior performance of the proposed  framework.

\subsection{Organization of the paper and notations}

The rest of the paper is organized as follows. Section~\ref{OF} introduces the optimization framework. Section~\ref{compute} develops the modified ADMM algorithm and states its convergence. Section~\ref{model_PF} describes the statistical properties of incorporating pairwise penalties in a scenario where latent variables form a small number of subgroups. Section~\ref{ex} contains extensive experiments on both synthetic and real datasets to demonstrate the superior performance of this general framework. Finally, we conclude in Section~\ref{sec:conclusions}. All proofs are delegated to the supplemental material.

\noindent {\bf Notations:} Let $\mathbf{A}=[A_{ij}]\in\mathbb{R}^{n\times m}$ be a $n\times m$ matrix and $\mathbf{x}=[x_1,...,x_n]^T\in\mathbb{R}^d$ be a $d$-dimensional vector. Recall the Frobenius norm $\|\mathbf{A}\|_F=(\sum_{ij}A_{ij}^2)^{1/2}$
 and the infinity norm $\|\mathbf{A}\|_{\infty}=\text{max}_{i,j}|A_{ij}|$. 
Define $\ell_1$ and $\ell_2$ norms as $\|\mathbf{x}\|_1=\sum_{i=1}^{d}|x_i|$ and $\|\mathbf{x}\|_2=(\sum_{i=1}^{d}x_i^2)^{1/2}$, respectively. The trace of a square matrix $\mathbf{A}$ is denoted as $\mathrm{tr}(\mathbf{A})$. Denote $|\mathcal{S}|$ as the cardinality of a set $\mathcal{S}$, and $\otimes$ as the Kronecker product. 
 
\section{Optimization Framework}\label{OF}

We develop a general optimization framework for incorporating pairwise penalties in low-rank matrix factorization, given in~(\ref{eq:genform4}), that allows a wide class of (non-)convex pairwise 
penalty functions $p\left( z, \gamma\right)$, characterized by the following assumption.
\begin{assumption}\label{c:1}
The penalty function $p(z,\gamma)$ satisfies
\begin{enumerate}
\item $p\left(z, \gamma\right)\geq 0$ for $z, \gamma \geq 0$,  
\item For any given vector $\mathbf{v}\in\mathbb{R}^{d}$, there exists a constant $\varsigma_0\geq 0$ independent of $\mathbf{v}$ such that for $\forall\,\varsigma>\varsigma_0$, $f(\mathbf{u})= \varsigma\|\mathbf{u}-\mathbf{v}\|_2^2+p\left(\|\mathbf{u}\|_2, \gamma\right)$ is strongly convex with respect to $\mathbf{u}\in\mathbb{R}^d$.
\end{enumerate} 
\end{assumption}
 
This assumption holds for a variety of (non-)convex pairwise penalty functions, including MCP~\cite{zhang2010nearly} and SCAD~\cite{fan2001variable}, both of which are highly popular as sparsity-promoting non-convex regularizers in the statistics literature. For example, Assumption~\ref{c:1} holds for the MCP regularizer
\begin{align}\label{mcp}
p_{\mathrm{MCP}}\left( z ,\gamma\right)=\gamma  |z| &-\frac{z^2}{2t}{\mathbf 1}\left(|z| \leq \gamma t\right)\nonumber
\\&+\left(\frac{t\gamma^2}{2}-\gamma |z|\right){\mathbf 1}\left( |z|> \gamma t\right), 
\end{align}
for $t>0$ by setting $\varsigma_0=1/(2t)$, where $\mathbf{1}(\cdot)$ denotes the indicator function. Moreover, we introduce the following M-type function:
\begin{align}\label{eq:Mtype}
 p_{\mathrm{M}}\left(z ,\gamma\right)= -\gamma\left( z^2 - 2 b|z|\right) \mathbf{1}\left( |z|<2b\right), \quad b>0,
\end{align}
which satisfies Assumption~\ref{c:1}  by letting $\varsigma_0=\gamma$. In addition, Assumption~\ref{c:1} also holds for all convex functions, e.g., the $\ell_1$ and the squared $\ell_2$ norms. Thus, our framework unifies and generalizes existing GRMF approaches.

\section{Algorithm and Convergence}\label{compute}
 For the class of optimization problems introduced in Section~\ref{OF}, we develop an efficient and general algorithm, dubbed $\bf{LLFMC}$, with a strong theoretical convergence guarantee. We begin with the following equivalent constrained form of the main problem~(\ref{eq:genform4}): 
\begin{align} \label{model} 
\min_{\mathbf{X,Y}} \,\, &\frac{1}{2}\|\Psi_{\Omega}(\mathbf{M}-\mathbf{X}\mathbf{Y}^T)\|^{2}_{F}+\frac{\alpha}{2}\left( \|\mathbf{X}\|^2_F+\|\mathbf{Y}\|^2_F\right)
\nonumber
\\& +\sum\limits_{l \in \varepsilon_{X}}w_{l}\hspace{0.04cm}p\left(\|\mathbf{p}_{l}\|_2,\gamma_{X}\right)+\sum\limits_{j \in \varepsilon_{Y}}u_{j}\hspace{0.04cm}p\left(\|\mathbf{q}_{j}\|_2,\gamma_{Y}\right)\nonumber \\ 
   \textrm{subject to}\,\, & \mathbf{p}_{l}=\mathbf{x}_{l_{1}}-\mathbf{x}_{l_{2}},\; \mathbf{q}_{j}=\mathbf{y}_{j_{1}}-\mathbf{y}_{j_{2}},
\end{align}
where the index set $ \varepsilon_{X} =\{l=\left(l_{1},l_{2}\right): w_{l_{1}l_{2}}>0,l_1<l_2\}$ and $\varepsilon_{Y}=\{ j=\left(j_{1}, j_{2}\right): u_{j_{1}j_{2}}>0, j_1<j_2\}$. It can be seen that for every $l=(l_1,l_2)\in  \varepsilon_{X}$, we can rewrite the first constraint in \eqref{model} as $\mathbf{p}_{l} = \mathbf{X}^T (\mathbf{e}_{l_1} - \mathbf{e}_{l_2})$, where $\mathbf{e}_k$ is the $k^{th}$ standard basis vector in $\mathbb{R}^n$. Let $ \mathbf{P}\in\mathbb{R}^{d\times\left|\varepsilon_{X}\right| }$ be the matrix whose columns are given by $\mathbf{p}_{l}^T$ with $l \in \varepsilon_{X}$, and let $\mathbf{E}_x\in\mathbb{R}^{n\times\left|\varepsilon_{X}\right|}$ be the matrix whose corresponding columns are given by $(\mathbf{e}_{l_1} - \mathbf{e}_{l_2})$. Symmetrically, we can define matrices $\mathbf{Q}\in\mathbb{R}^{d\times\left|\varepsilon_{Y}\right|}$ and $\mathbf{E}_y\in\mathbb{R}^{m\times\left|\varepsilon_{Y}\right|}$. The constraint in \eqref{model} can be rewritten in a matrix form as
\begin{align}\label{eq:constraint2}
\mathbf{P}-\mathbf{X}^T{\mathbf{E}}_x=\mathbf{0},\,\mathbf{ Q-Y}^T\mathbf{E}_y=\mathbf{0}.
\end{align}  
Let $\mathcal{L}_{\eta}\left( \mathbf{P, Q,X,Y,\Lambda,V}\right)$ be the augmented Lagrangian function of the problem~(\ref{model}), which is given as
\begin{align} 
 \mathcal{L_{\eta}}&=\frac{1}{2}\|\Psi_{\Omega}(\mathbf{M-XY}^T)\|^{2}_{F}+\frac{\alpha}{2}(\|\mathbf{X}\|_{F}^{2}+\|\mathbf{Y}\|_{F}^{2}) \nonumber
 \\&+\tr\big(\mathbf{\Lambda}^T\big(\mathbf{ P - X}^T\mathbf{E}_x\big) \big)+\tr\big(\mathbf{V}^T\big(\mathbf{ Q - Y}^T\mathbf{E}_y\big) \big)\nonumber
\\&+\frac{\eta}{2}\big(\|\mathbf{ P-X}^T\mathbf{E}_x\|_F^2+\|\mathbf{ Q-Y}^T\mathbf{E}_y\|_F^2\big)\nonumber
\\&+\sum\limits_{l \in \varepsilon_{X} }w_{l}\hspace{0.04cm}p\left(\|\mathbf{p}_{i}\|_2,\gamma_{X}\right)+ \sum\limits_{j \in \varepsilon_{Y}} u_{j}\hspace{0.04cm}p\left(\|\mathbf{q}_{j}\|_2,\gamma_{Y}\right), \label{lagrangian}
\end{align}
where $\mathbf{\Lambda} \in \mathbb{R}^{d\times \left|\varepsilon_{X}\right|}$ and $ \mathbf{V}\in \mathbb{R}^{d\times \left|\varepsilon_{Y}\right|}$ are the dual variables.
 
We now describe our 2-step algorithm, LLFMC, based on a modified ADMM.
\begin{enumerate}
\item {\bf Step 1.} Define an undirected graph $\mathcal{G}_X=(V_X, \varepsilon_{X})$ with $V_X=\{1,2,...,n\}$. We first use the standard depth-first-search algorithm to find cycles. Then, for each cycle of $\mathcal{G}_X$, we randomly cut one edge $l$ off by letting $w_l=0$. As a result, the graph $\mathcal{G}_X$ becomes acyclic. 
Similar operations can be applied for $\mathcal{G}_Y$.

\item{\bf Step 2.}  
In each iteration $k$, we do the following updates:\\
\noindent a) \text{Primal updates:} 
\begin{subequations}\label{modified_ADMM}
\begin{align}
\mathbf{P}^{k+1} &=\arg\min\limits_{\mathbf{ P}}{\mathcal{L}_{\eta}\big(\mathbf{ P, Q}^k,\mathbf{X}^k,\mathbf{Y}^k,\mathbf{\Lambda}^k,\mathbf{V}^k\big)},\label{modified_ADMM_p}
\\ \mathbf{ Q}^{k+1}&=\arg\min\limits_{\mathbf{Q}}{\mathcal{L}_{\eta}\big(\mathbf{ P}^{k+1},\mathbf{ Q},\mathbf{X}^k,\mathbf{Y}^k,\mathbf{\Lambda}^k,\mathbf{V}^k\big)},\label{modified_ADMM_q}
\\ \mathbf{X}^{k+1}&=\arg\min\limits_{\mathbf{X}}\Big(\mathcal{L}_{\eta}\big(\mathbf{ P}^{k+1},\mathbf{ Q}^{k+1},\mathbf{X,Y}^k,\mathbf{\Lambda}^k,\mathbf{V}^k\big)\nonumber
\\&\hspace{2.5cm}+\frac{\alpha_x}{2}\|\mathbf{X-X}^k\|_F^2\Big),\label{modified_ADMM_x}\\
\mathbf{Y}^{k+1}&=\arg\min\limits_{\mathbf{Y}}\Big(\mathcal{L}_{\eta}\big(\mathbf{P}^{k+1},\mathbf{ Q}^{k+1},\mathbf{X}^{k+1},\mathbf{Y,\Lambda}^k,\mathbf{V}^k\big)\nonumber
\\&\hspace{2.5cm}+\frac{\alpha_y}{2}\|\mathbf{Y-Y}^k\|_F^2\Big),\label{modified_ADMM_y}
\end{align}
\end{subequations}
\noindent b) \text{Dual updates:}
\begin{align}\label{d_upp}
&\mathbf{\Lambda}^{k+1}=\;\mathbf{\Lambda}^{k}+\eta\big(\mathbf{ P}^{k+1}-(\mathbf{X}^{k+1})^T\mathbf{E}_x\big),\nonumber
\\&\mathbf{V}^{k+1}=\;\mathbf{V}^{k}+\eta\big(\mathbf{ Q}^{k+1}-(\mathbf{Y}^{k+1})^T\mathbf{E}_y\big). 
\end{align}

\end{enumerate}

Some discussions are in order. Step 1 is necessary to address the row-rank deficiency issue of the matrices $\mathbf{E}_x$ and $\mathbf{E}_y$ \footnote{Since the summation of the entries in $\mathbf{E}_x$ is equal to $0$, the row rank of $\mathbf{E}_x$ in the constraint~(\ref{eq:constraint2}) is at most $n-1$. A similar argument holds for $\mathbf{E}_y$.} that prevents the application of convergence guarantees of standard ADMM, e.g. \cite{wang2018convergence,hong2016convergence,wang2015global}. As shall be seen later, this step verifies Property~\ref{graph} in Section~\ref{conver}, which is needed to obtain the convergence guarantee of our algorithm.

 Compared to standard ADMM, our algorithm further introduces Bregman divergences in Step 2 to the update of $\mathbf{X}$ and $\mathbf{Y}$ in \eqref{modified_ADMM_x} and \eqref{modified_ADMM_y} to improve the convergence performance. As shall be seen later, these modifications guarantee 
sufficient descents of $\mathcal{L}_{\eta}$  during the subproblems \eqref{modified_ADMM_x} and \eqref{modified_ADMM_y} in each iteration.
 
 In Fig.~\ref{fig:convergence}, we empirically investigate the impacts of  Step 1 and the inclusion of Bregman divergence in Step 2 on the convergence performance of our algorithm. 
We can observe that without Bregman divergence in Step 2, ADMM will diverge after around 600 iterations, whereas ADMM with Bregman divergence in Step 2 has a better convergence performance as well as a lower relative error. In addition, without Step 1, ADMM still has a good convergence performance but achieves a larger relative error than that with Step 1. In sum, the Bregman divergence term in Step 2 indeed contributes to the better convergence performance of our modified ADMM algorithm, while Step 1 helps to attain a lower relative error but does not affect the convergence behavior much. 
	
\begin{figure}[t]
\centering
\includegraphics[width=0.5\textwidth]{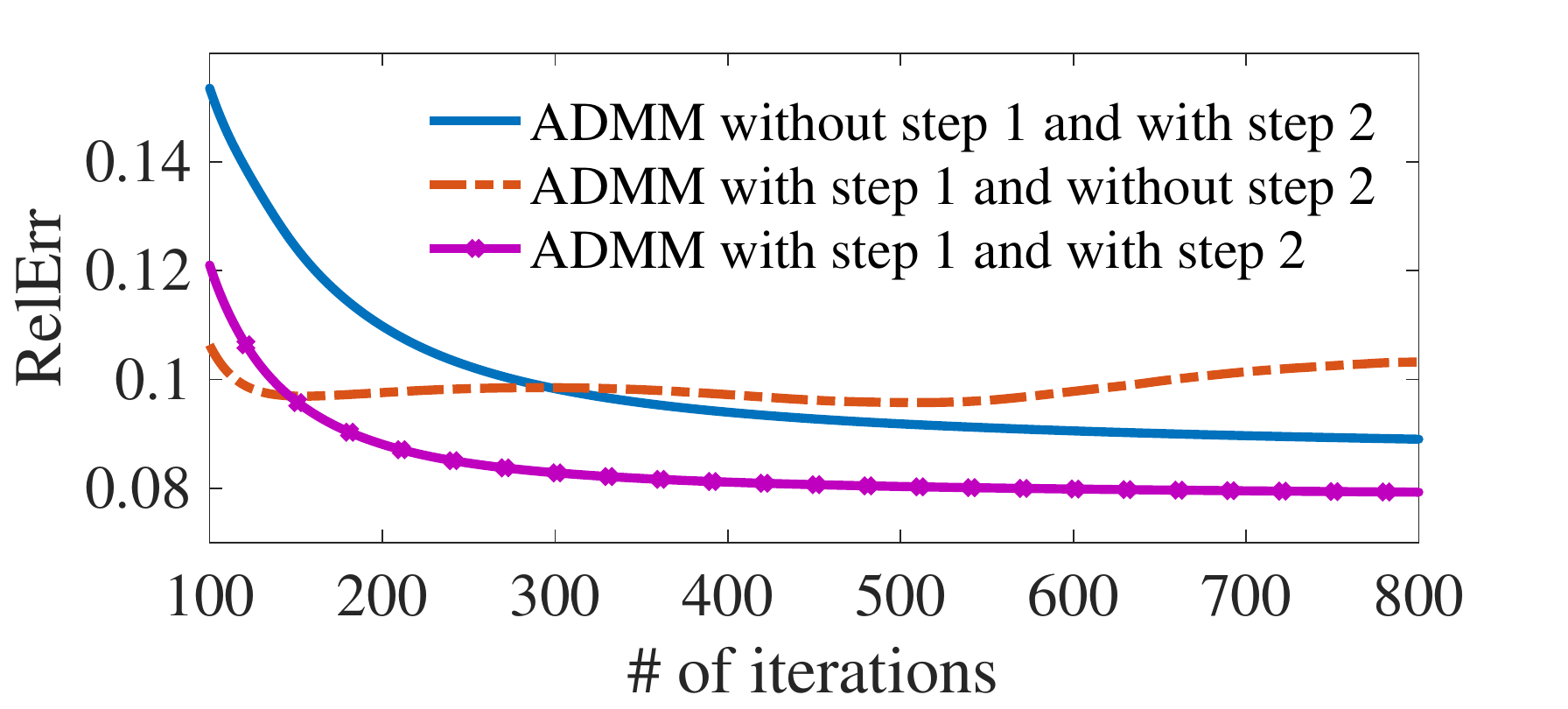}
\caption{Illustration of the impact of Step 1 and the inclusion of Bregman divergence Step 2 on the convergence performance. We use MCP as the pairwise penalty function, and randomly pick $30\%$ of the pixels of a $512\times512$ image \textit{Lena} to generate the partially observed matrix $\Psi_{\Omega}(\mathbf{M})$. The relative error is defined by ${\rm RelErr}=\|\mathbf{{\widehat M}}-\mathbf{M}\|_{F}/\|\mathbf{M}\|_{F}$, where $\mathbf{\widehat M}$ is the estimate.}\label{fig:convergence}
\end{figure}

\begin{Remark}
Compared to the Bregman ADMM~\cite{wang2014bregman}, we do not apply Bregman divergence to $\mathbf{P}$ and $\mathbf{Q}$ in order to save $O\left((\left|\varepsilon_{X}\right|+\left|\varepsilon_{Y}\right|)\times d\right)$ space for tracking $\mathbf{P}^k$ and $\mathbf{Q}^k$.  
 \end{Remark}

\subsection{A Fast Inexact Solver via CG}
We now discuss how to solve the subproblems in Step 2 efficiently. Based on~\eqref{lagrangian}, updating $\mathbf {P}$ in~(\ref{modified_ADMM_p}) is to solve the following proximal map for each of its column $\mathbf{p}_{l}^T$ in parallel:
\begin{align}
\mathbf{p}_{l}^{k+1}=\arg\min\limits_{\mathbf{p}_{l}}&\Big(\frac{1}{2}\|\mathbf{p}_{l}^T -\left((\mathbf{X}^{k})^T\mathbf{E}_x[l]-{\eta}^{-1}\mathbf{\Lambda}^{k}[l]\right)\|_{2}^{2}\nonumber
\\&+ w_{l}\eta^{-1}p\left(\|\mathbf{p}_{l}\|_2,\gamma_{X}\right)\Big), \label{pl}
\end{align}
where $\mathbf{E}_x[l]$ and $\mathbf{\Lambda}^{k}[l]$ are the column of $\mathbf{E}_x$ and $\mathbf{\Lambda}^{k}$ corresponding to  $\mathbf{p}_{l}$, respectively. 
In practice, we choose $\eta>2\varsigma_0 \max_{l \in \varepsilon_X } w_{l}$, which, based on  Assumption~\ref{c:1}, implies that the problem in~(\ref{pl}) has a unique solution. Notably, this solution has  simple and explicit analytical forms for the penalty functions such as SCAD, MCP and the M-type function. We update $\mathbf{Q}$ in a similar way by solving the proximal map for each column $ \mathbf{q}_{j}^T$, given as
 \begin{align}\label{qj}  
 \mathbf{q}_{j}^{k+1}=\arg\min_{\mathbf{q}_{j}} &\Big(\frac{1}{2}\|\mathbf{q}_{j}^T-\big((\mathbf{Y}^{k})^T\mathbf{E}_y[j]-\eta^{-1}\mathbf{V}^{k}[j] \big)\|_{2}^{2} \nonumber \\
          &  +u_{j}\eta^{-1}\hspace{0.04cm}p\left(\|\mathbf{q}_{j}\|_2,\gamma_{Y}\right) \Big),
 \end{align}
where $\mathbf{E}_y[j]$ and $\mathbf{V}^{k}[j]$ are the column of $\mathbf{E}_y$ and $\mathbf{V}^{k}$ corresponding to $\mathbf{q}_{j}$, respectively.
                 
Based on~(\ref{modified_ADMM_x}), updating $\mathbf{X}$ is to minimize 
\begin{align}
&F(\mathbf{X}) =\frac{1}{2}\sum_{(i,j)\in\Omega} \left(M_{ij}-\mathbf{x}_i(\mathbf{y}^k_j)^T\right)^2+\frac{\eta}{2}\|\mathbf{P}^{k+1}-\mathbf{X}^T\mathbf{E}_x\|_F^2\nonumber
 \\&+\tr\left( (\mathbf{\Lambda}^k )^T\left(\mathbf{ P}^{k+1} - \mathbf{X}^T\mathbf{E}_x\right) \right)+\frac{\alpha}{2}\|\mathbf{X}\|_{F}^{2}+\frac{1}{2}\|\mathbf{X}-\mathbf{X}^k\|_F^2. \nonumber
\end{align} 
Recognizing that $F(\mathbf{X})$ is quadratic in $\mathbf{X}$, the optimizer can be found by solving a linear system of equations, where $\mbox{vec}(\mathbf{X}^T)$ satisfies
\begin{align}\label{inverse}
\left[\mathbf{G}^k_y+(\eta\mathbf{E}_x\mathbf{E}_x^T+(\alpha+1)\mathbf{I}_n)\otimes \mathbf{I}_d \right]\text{vec}(\mathbf{X}^T)=\mathbf{c},
\end{align}
where 
\begin{align}\label{cg}
\mathbf{c}=(\mathbf{b}^k_y)^T+{\normalfont \text{vec}}((\mathbf{X}^k)^T+\eta\mathbf{P}^{k+1}\mathbf{E}_x^T+\mathbf{\Lambda}^k\mathbf{E}_x^T).
\end{align}
Here, $\mathbf{G}_y^k=\text{diag}\left(\mathbf{G}_1,\cdots,\mathbf{G}_n\right) $ is a block diagonal matrix with the $i^{th}$ block given as $\mathbf{G}_i=\sum_{j\in\Omega_i}(\mathbf{y}^k_j)^T(\mathbf{y}_j^k)\in\mathbb{R}^{d\times d}$, where $\Omega_i=\{ j:(i,j)\in\Omega\}$ is the index set of the observed entries in the $i^{th}$ row. The row vector $\mathbf{b}^k_y=[\mathbf{b}_1,\cdots, \mathbf{b}_n]$ concatenates the vectors $\mathbf{b}_i=\sum_{j\in\Omega_i}M_{ij}\mathbf{y}^k_j$,  
with $\mathbf{y}_j^k$ being the $j^{th}$ row of the matrix $\mathbf{Y}^{k}$.

 \begin{algorithm}[t]
   \caption{Learning latent features with pairwise penalties in low-rank matrix factorization (LLFMC)}
   \label{alg:ours}
\begin{algorithmic}[1]
   \STATE {\bfseries Input:} $\Psi_{\Omega}\left(\mathbf{M}\right)\in\mathbb{R}^{n\times m}$, rank $d\in\mathbb{N}^{+}$, $\rm MaxIter>0$, $\rm tol_1,tol_2>0.$ 
   \STATE {\bfseries Initialize:} $\gamma_{X}, \gamma_{Y}, \alpha,\eta$. Set $\mathbf{X}^1$, $\mathbf{Y}^1$, $\mathbf{\Lambda}^1$, $\mathbf{V}^1$ as random matrices.
   \FOR{$k=1$ {\bfseries to} $\rm MaxIter$}
    \FOR{$l \in \varepsilon_{X}$ }
   \STATE{Compute $\mathbf{p}_{l}^{k+1}$  by solving the proximal map~\eqref{pl};}
   \ENDFOR
   \FOR{$j\in \varepsilon_{Y}$}
   \STATE {Compute $\mathbf{q}_{j}^{k+1}$ by solving the proximal map~\eqref{qj};} 
   \ENDFOR 
   \STATE {Update $\mathbf{X}^{k+1},\mathbf{Y}^{k+1}$ by solving~(\ref{inverse}) through CG};
   \STATE {$\mathbf{\Lambda}^{k+1}=\mathbf{\Lambda}^{k}+\eta\left(\mathbf{ P}^{k+1}-(\mathbf{X}^{k+1})^T\mathbf{E}_x\right)$;}
   \STATE {$\mathbf{V}^{k+1}=\mathbf{V}^{k}+\eta\left(\mathbf{ Q}^{k+1}-(\mathbf{Y}^{k+1})^T\mathbf{E}_y\right)$;}
   \IF{$D_k<\text{tol}_1$ or $\left|D_k-D_{k+1}\right|<\text{tol}_2$} \STATE break and output $\mathbf{X}^{k+1}$ and $\mathbf{Y}^{k+1}$;
   \ENDIF
   \ENDFOR
   \STATE {\bfseries Output:} $\mathbf{\widehat M}=\mathbf{X}^{k+1}(\mathbf{Y}^{k+1})^T$.
\end{algorithmic}
\end{algorithm}

Unfortunately, solving~(\ref{inverse}) requires inverting an $nd\times nd$ matrix, which is computationally demanding when dealing with large datasets. To address this issue, we apply the standard CG to directly minimize $f(\text{vec}(\mathbf{X}^T))$, where 
the function $f({\normalfont \text{vec}}(\mathbf{X}^T))$ is a re-write of $F(\mathbf{X})$ in terms of its vectorized argument $ \text{vec}(\mathbf{X}^T)$, defined as 
\begin{align*}
f(&\mathbf{s})= \frac{1}{2}\mathbf{s}^T\left[\mathbf{G}^k_y+(\eta\mathbf{E}_x\mathbf{E}_x^T+(\alpha+1)\mathbf{I}_n)\otimes \mathbf{I}_d\right] \mathbf{s}- \mathbf{c}^T\mathbf{s}.
\end{align*}

The most expensive part in each CG iteration is the Hessian-vector multiplication $\nabla^2f(\mathbf{s})\mathbf{s}$. Using the identity $(\mathbf{B}^T\otimes \mathbf{I})\text{vec}(\mathbf{S})=\text{vec}(\mathbf{SB})$, we have 
\begin{align}\label{multi}
\nabla^2f(\mathbf{s})\mathbf{s}=\mathbf{G}_y^k\mathbf{s}+\text{vec}(\eta\mathbf{S} \mathbf{E}_x\mathbf{E}_x^T+(\alpha+1)\mathbf{S}),
\end{align}
with $\mathbf{S}=[\mathbf{s}_1,\cdots,\mathbf{s}_n] \in\mathbb{R}^{d\times n}$ and $\mathbf{s}=\text{vec}(\mathbf{S})$. 
Since $\mathbf{G}_y^k$ is a block diagonal matrix, to compute $\mathbf{G}_y^k\mathbf{s}$ it suffices to obtain $\mathbf{G}_i\mathbf{s}_i$ for $1\leq i\leq n$,  
which can be  computed in $O(|\Omega_i|d)$ time by using   $\mathbf{G}_i\mathbf{s}_i=\sum_{j\in\Omega_i}(\mathbf{y}_j^k\mathbf{s})(\mathbf{y}^k_j)^T$.
 
The time complexity for computing a single CG iteration is $O((|\Omega|+nnz(\mathbf{E}_x\mathbf{E}_x^T))\times d)$, where $nnz(\cdot)$ is the number of nonzeros. For our algorithm, we use a $k_w$-nearest neighbor method to select the weights $w_{ij}$, as introduced in section~\ref{ss:weight}. Thus, 
using the structure of the matrix $\mathbf{E}_x$, we have an upper bound for $nnz(\mathbf{E}_x\mathbf{E}_x^T)$, as shown in the following proposition.
\begin{prop}\label{nnzE}
$nnz(\mathbf{E}_x\mathbf{E}_x^T)\leq n (k_w+1)$. 
\end{prop}
The proof can be found in Appendix~\ref{apen:A}.
In most real-world applications, $k_w$ is small and satisfies $n(k_w+1)<|\Omega|$.  
Thus, 
each CG iteration can be computed in $O(|\Omega|\times d)$ time. 
We stop CG if the approximation error of \eqref{inverse} is small enough: \begin{align}\label{stop:ADM}
& \left\|\left[\mathbf{G}^k_y+(\eta\mathbf{E}_x\mathbf{E}_x^T+(\alpha+1)\mathbf{I}_n)\otimes \mathbf{I}_d \right]\text{vec}((\mathbf{X}^t)^T)-\mathbf{c} \right\|_2  \nonumber \\
 & \leq 10^3\sqrt{dn}/(k+1)^{1.2},      
 \end{align}
using the output of the $t^{th}$ CG iteration $\mathbf{X}^t$. Note that we do not require CG  to fully converge. Instead, we only need a very small number (empirically $\leq 5$) of CG iterations to obtain an \textit{inexact} solution $\mathbf{X}^{k+1}$ that satisfies the stopping criteria \eqref{stop:ADM}, which still guarantees the convergence for our algorithm, as shown in Theorem~\ref{th:inexact} below. This step further speeds up our algorithm. 

The main steps are summarized in Algorithm~\ref{alg:ours}, 
where $D_k=\|\mathbf{X}^k-\mathbf{X}^{k+1}\|_F/(2\sqrt{dn})+\|\mathbf{Y}^k-\mathbf{Y}^{k+1}\|_F/(2\sqrt{dm})$.
In our implementation,  we choose $\textrm{tol}_1=10^{-1}$ and $\textrm{tol}_2=10^{-4}$.

\begin{Remark} The Graph Regularized
Alternating Least Squares (GRALS)~\cite{rao2015collaborative} algorithm is a GRMF method that is most relevant to ours, and can be regarded as solving a special case of our framework \eqref{eq:genform3} when the objective function is convex and \textit{differentiable} with respect  to $\mathbf{X}$ or $\mathbf{Y}$ when fixing the other, and therefore  much easier to optimize. In contrast, we propose an algorithm to solve the much more general optimization problem that is non-convex and non-smooth.
\end{Remark}

\subsection{Convergence Guarantees}\label{conver}   
This section proves the convergence of our algorithm.
Let $\mathbf{U}^k=\left(\mathbf{ P}^{k},\mathbf{ Q}^{k},\mathbf{X}^{k},\mathbf{Y}^{k},\mathbf{\Lambda}^k,\mathbf{V}^k\right)$ be the exact solutions of the subproblems in~(\ref{modified_ADMM}). 
We first prove the convergence of the sequence $\left\{\mathbf{U}^k\right\}$. Then, we extend this convergence result to Algorithm~\ref{alg:ours} that solves~(\ref{inverse}) (or equivalently, \eqref{modified_ADMM_x} and \eqref{modified_ADMM_y}) inexactly. 

Our convergence analysis adopts the following assumption on the boundedness of the generated sequence $\left(\mathbf{X}^k,\mathbf{Y}^k\right)$.
\begin{assumption}\label{bound}
The sequence $\left(\mathbf{X}^k,\mathbf{Y}^k\right)$ generated by~(\ref{modified_ADMM}) is bounded, which means that there exist constants $M_x,M_y$  such that $\|\mathbf{X}^k\|^2_F\leq M_x$, $\|\mathbf{Y}^k\|^2_F\leq M_y$ for all $k$. 
\end{assumption}
 
We note that directly proving the finiteness of the iterates is very challenging. 
Compared with previous analysis \cite{wang2015global,wang2014bregman}, where the boundedness of the iterates is established typically based on stronger assumptions, e.g. the Lipschitz property of the gradient of the augmented Lagrangian function, which do not hold in our problem. Nonetheless, such an assumption does not seem uncommon in the optimization literature, particularly when dealing with functions that are non-smooth and non-convex, e.g. \cite{davis2018stochastic,duchi2018stochastic}. In practice, we find that the generated sequence by our algorithm is well-bounded, and hence this assumption is rather mild.

%	We find that proving the finiteness of the function value sequence is very challenging compared to existing analysis for ADMM type of algorithms. Specifically, we note that previous analysis~\cite{wang2015global,wang2014bregman} does not assume such a boundedness condition, but prove it under a stronger condition. Let us take \cite{wang2015global} for example, and its proof outline is as follows. First, they prove that the augmented Lagrangian function is non-increasing and lower-bounded. Then, based on an assumption that the objective function is coercive, i.e., $f(x)\rightarrow\infty$ if $x\rightarrow\infty$ for a function $f$, they enable to prove the boundedness of the generated sequence. However, their analysis needs one key assumption that the augmented Lagrangian function is gradient-Lipschitz in terms of one variable ($\mathbf{X}$ in our case). This condition does not hold for our case because one can check that the gradient-Lipschitz constant of our augmented Lagrangian function in terms of $\mathbf{X}$ is proportional to $\mathbf{Y}$. As a result, this gradient-Lipschitz constant cannot guarantee to be finite, and hence proving the boundedness of the generated sequence is very challenging. However, in all experiments, we find that the generated sequence is well bounded, and hence this assumption may be not  that strong. 

Following~\cite{he2015full,deng2017parallel}, we use the quantity $\|\mathbf{U}^{k+1}-\mathbf{U}^k\|_F^2$ as a measure of the convergence rate for the sequence $\{\mathbf{U}^{k}\}$. We have the following theorem, whose proof can be found in Appendix~\ref{apen:C}.
\begin{theorem}\label{th:exact}
	Suppose Assumptions~\ref{c:1} and~\ref{bound} are satisfied. Choosing a sufficiently large $\eta$, we have $\sum_{k=1}^\infty \| \mathbf{U}^{k+1} -\mathbf{U}^{k} \|_F^2<\infty$. Furthermore, if the penalty function $p(\cdot, \gamma)$ is semi-algebraic,  then $\sum_{k=1}^\infty \|\mathbf{U}^{k+1}-\mathbf{U}^k\|_F<\infty$ and 
the sequence $\left\{\mathbf{U}^k\right\}$ converges to a stationary point of $\mathcal{L}_{\eta}$.
 
\end{theorem}
 
The first result $\sum_{k=1}^\infty \| \mathbf{U}^{k+1} -\mathbf{U}^{k} \|_F^2<\infty$ guarantees a subsequential convergence for our proposed algorithm under general penalty functions including MCP and SCAD. Furthermore, for all semi-algebraic penalty functions, which include all real polynomial functions and $\ell_p$-norm with any $p>0$ (see Examples 2 and 4 in~\cite{bolte2014proximal}), we establish the convergence of the whole sequence $\{\mathbf{U}^k\}$ by exploring the Kurdyka-{\L}ojasiewicz (K\L) property of the augmented Lagrangian function in~\eqref{lagrangian}. 

We further extend the convergence guarantee in Theorem~\ref{th:exact} to the scenario  when \eqref{modified_ADMM_x} and \eqref{modified_ADMM_y}) are solved inexactly. Let $\mathbf{\widehat U}^k=({\mathbf{P}}^{k},{\mathbf{Q}}^{k},\mathbf{\widehat X}^{k},\mathbf{\widehat Y}^{k},{\mathbf{\Lambda}}^k,{\mathbf{V}}^k)$ be the sequence generated by Algorithm~\ref{alg:ours} using inexact solutions $\mathbf{\widehat X}^{k}$ and $\mathbf{\widehat Y}^{k}$ to \eqref{modified_ADMM_x} and \eqref{modified_ADMM_y}. Define 
$$\mathbf{t}_x^k=\left[\mathbf{G}^k_y+(\eta\mathbf{E}_x\mathbf{E}_x^T+(\alpha+1)\mathbf{I}_n)\otimes \mathbf{I}_d \right]\text{vec}((\mathbf{\widehat X}^{k})^T)-\mathbf{c}$$ for $k\geq 1$, which measures the approximation error of solving \eqref{inverse}. In a symmetric way,  we can define $\mathbf{t}_y^k$. Based on (\ref{stop:ADM}), we have $\sum_{k=1}^{\infty} \|\mathbf{t}_x^k\|_2<\infty$ and $\sum_{k=1}^{\infty} \|\mathbf{t}_y^k \|_2<\infty$. We have the following theorem, whose proof can be found in Appendix~\ref{apen:D}.
\begin{theorem}\label{th:inexact}
Suppose Assumption~\ref{c:1} and Assumption~\ref{bound} hold with $(\mathbf{\widehat X}^{k},\mathbf{\widehat Y}^{k})$. Choosing a sufficiently large $\eta$, we have $\sum_{k=1}^\infty \|\mathbf{\widehat U}^{k+1}-\mathbf{\widehat U}^k\|_F^2 < \infty$. Furthermore, if the penalty function $p(\cdot, \gamma)$ is semi-algebraic, then $\sum_{k=0}^\infty \|\mathbf{\widehat U}^{k+1}-\mathbf{\widehat U}^k\|_F<\infty$ and the sequence $ \{\mathbf{\widehat U}^k \}$  
converges to a stationary point of $\mathcal{L}_{\eta}$. Furthermore, the best running convergence rate of Algorithm~\ref{alg:ours} is $o(1/k)$.
\end{theorem} 
Different from using the exact solution in Theorem~\ref{th:exact},  
Theorem~\ref{th:inexact} 
only requires $\mathbf{\widehat X}^k$ to satisfy $\sum_{k=1}^{\infty}\|\mathbf{t}_x^k\|_2<\infty$ (similarly for $\mathbf{\widehat Y}^k$). This extension allows more efficient methods to find inexact solutions of~(\ref{inverse}), e.g., the CG approach used in our algorithm.

%Following~\cite{he2015full,deng2017parallel}, we use the quantity $\|\mathbf{\widehat U}^{k+1}-\mathbf{\widehat U}^k\|_F^2$ as a measure of the convergence rate for the sequence $\big(\mathbf{\widehat U}^{k}\big)$.

\section{Statistical Properties}\label{model_PF}
 
Fully characterizing the statistical properties of the proposed
class of penalty functions in Section~\ref{OF} is challenging.
Instead, we restrict to an important class, where the latent vectors can be divided into subgroups, and rigorously derive the estimation error bound for this class. 
More importantly, our theoretical analysis motivates two interesting directions for our framework, i.e., subgroup identification and adaptive weights.  
We first introduce the subgroup-based model in Section~\ref{ss:subgroup}. Then, the estimation error is characterized in Section~\ref{com_bound}.
This characterization directly motivates us the use of adaptive weights in Section~\ref{ss:weight}.    

\subsection{Subgroup-based Model}\label{ss:subgroup}
Assume that there exists a ground truth low-rank matrix $\mathbf{M}^*\in\mathbb{R}^{n\times m}$, which can be factored as $\mathbf{M^*}=\mathbf{X^*Y^*}^T$, where $\mathbf{X^*} \in\mathbb{R}^{n\times d}$, and $\mathbf{Y^*} \in\mathbb{R}^{m\times d}$. The observations are given as $\Psi_{\Omega}(\mathbf{M})=\Psi_{\Omega}(\mathbf{M^*}+\mathbf{N})$, which are corrupted by the additive noise $\mathbf{N}$ composed of i.i.d. Gaussian entries $\mathcal{N}(0,\sigma^2)$ with zero mean and variance $\sigma^2$. For an integer $s$ with $s<nm$, assume that each index pair $(i,j)\in[n]\times[m]$ is included in the observed index set $\Omega$ independently with probability $s/(nm)$. The elements of the matrix $\Psi_{\Omega}(\mathbf{M})$ are independent conditioned on the set $\Omega$. Assume $\|\mathbf{X}^*\|_{\infty}\leq C_x$, $\|\mathbf{Y}^*\|_{\infty}\leq C_y$, and $\|\mathbf{M^*}\|_{\infty}\leq C_m$, respectively.

 The feature vectors $\{\mathbf{x}^*_i\}$ and $\{\mathbf{y}^*_j\}$ form some latent subgroups, where two features $\mathbf{x}^*_i$ and $\mathbf{x}^*_j$ (respectively $\mathbf{y}^*_s$ and $\mathbf{y}^*_t$) are considered in the same subgroup if  $\mathbf{x}^*_i=\mathbf{x}^*_j$ (respectively $\mathbf{y}^*_s=\mathbf{y}^*_t$). One key difference with the existing group based models, e.g., the group sparsity model~\cite{kim2012group}, is that
the subgroup structure in our model is not assumed to be known a priori. Studies~\cite{delporte2013socially} have shown that the users with similar types in recommender systems often have the same features, implying a natural subgroup structure. 
Nevertheless, our subgroup-based model can be fairly general.
For example, the number of subgroups can even be as large as the matrix dimension.  

To facilitate analysis, we introduce additional useful notations. 
Let $\mathcal{G}(\mathbf{X})=(\mathcal{G}^x_1, \cdots, \mathcal{G}^x_{k_x})$ and $\mathcal{G}(\mathbf{Y})=(\mathcal{G}^y_1, \cdots, \mathcal{G}^y_{k_y})$ be two sets of 
mutually exclusive partitions of the indices $\{1,..., n\}$ and $\{1,..., m\}$, which satisfy $\mathbf{x}_i=\mathbf{x}_j$ for $\forall\; i, j \in \mathcal{G}^x_u$, $1\leq u \leq k_x$ and  $\mathbf{y}_s=\mathbf{y}_t$ for $\forall\; s, t \in \mathcal{G}^y_v$, $1\leq v \leq k_y$.  
We  use $|\mathcal{G}(\mathbf{X})|=k_x$ and $|\mathcal{G}(\mathbf{Y})|=k_y$ to denote the number of 
subgroups of $\{\mathbf{x}_i\}$ and $\{\mathbf{y}_j\}$, respectively. 

\subsection{Estimation Error Bound}\label{com_bound} 

We first introduce the complexity-regularized maximum likelihood estimator \cite{li2000mixture,kolaczyk2005multiscale}, and show it is a special case of our framework. Then,  we prove that this estimator achieves a lower error bound
compared to the standard trace-norm regularized matrix completion without pairwise penalties.
 
Consider the following complexity-regularized maximum likelihood estimator
\begin{align}\label{esjky}
\mathbf{\widehat M} & =\arg\min\limits_{\mathbf{H} \in\mathcal{H}} \|\Phi_{\Omega}(\mathbf{M-H})\|_F^2  +\lambda(|\mathcal{G}(\mathbf{X})|+|\mathcal{G}(\mathbf{Y})|) ,
\end{align}
where the finite set $\mathcal{H}$ is given as
\begin{align}\label{mathM}
\mathcal{H}\overset{\Delta}{=}\big\{\mathbf{H=XY}^T:\mathbf{X}\in\mathcal{X},\mathbf{Y}\in{\mathcal{Y}}, \|\mathbf{H} \|_{\infty}\leq C_m\big\},
\end{align} 
with the set $\mathcal{X}$ in~(\ref{mathM}) defined by
\begin{align}\label{mathX}
\mathcal{X}\overset{\Delta}{=}\big\{\mathbf{X}\in\mathbb{R}^{n\times d}: &X_{ij}\in\{-C_x+\frac{2C_x} {K}t, t=0,...,K-1\} \big\},
\end{align}
where $K=2^{\lceil \mu\log_2\,(n \vee m)\rceil}$, $n\vee m = \max(n,m)$ and $\mu>1$. 
In a symmetrical way, we can also define the set $\mathcal{Y}$.

Using the following proposition, proved in Appendix~\ref{apen:E}, we show that the estimator~(\ref{esjky}) is a special case of our framework~(\ref{eq:genform4}).
 \begin{prop}\label{cl1}
 Let $w_{ij}=1/\left(K_i(n-K_i)\right)$, where  $K_i$ is the total number of feature vectors $\mathbf{x}_s$ that satisfy $\|\mathbf{x}_s-\mathbf{x}_i \|_2<2C_x/K$. Then, we have for all  $\mathbf{X}\in\mathcal{X}$, 
\begin{align}\label{gxgy}
|\mathcal{G}(\mathbf{X})|=\sum_{i,j}w_{ij} \Upsilon(\|\mathbf{x}_i-\mathbf{x}_j\|_2),
\end{align}
where the indicator function $\Upsilon(z)=0$ if $z<2C_x/K$ and $1$ otherwise.  A symmetrical result holds for $|\mathcal{G}(\mathbf{Y})|$. 
\end{prop}  
Moreover, the constraints $\|\mathbf{X}\|_{\normalfont{\infty}}\leq C_x$ and $\|\mathbf{Y} \|_{\normalfont{\infty}}\leq C_y$ in~(\ref{mathX}) play a similar role as the regularizer $\alpha(\mathbf{\|X\|}_F^2+\mathbf{\|Y\|}_F^2)/2$ in our framework \eqref{eq:genform4}.
 In the following theorem, we provide an error bound for the estimator~(\ref{esjky}).
\begin{theorem}\label{th3}
Set $\lambda\geq 8\mu d  \left(\sigma^2+4C_m^2/3\right) \log \left(n \vee m\right)$, then the estimator~(\ref{esjky}) satisfies the error bound
\begin{align}\label{bbound}
&\frac{1}{nm}\mathbb{E}\left[\|\mathbf{\widehat M}-\mathbf{M}^*\|_F^2\right] \leq \frac{C_1\log s}{s}+ \frac{dC_2}{(\mu\log(n \vee m))^2}+\nonumber
\\&\left(\frac{6\lambda}{\log(n \vee m)}+\mu C_3\right)
\bigg(\frac{(m+n+2)}{s}\log(n \vee m)+\nonumber
\\&\;\frac{(|\mathcal{G}(\mathbf{X}^*)|+|\mathcal{G}(\mathbf{Y}^*)|)d}{s}\log(n \vee m)\bigg),
\end{align}
where the expectation is with respect to the joint distribution of $\left( \Omega, \Phi_{\Omega}(\mathbf{M}) \right)$ and $C_1,C_2,C_3$ are positive constants related to $C_x,C_y$ and $C_m$.
\end{theorem} 
We refer to  Appendix~\ref{apen:F} for its proof. We compare the error bound~(\ref{bbound}) with the one for the trace-norm regularized matrix completion obtained by~\cite{koltchinskii2011nuclear}, which assumes the observed entries are corrupted by additive Gaussian noise.  
Instantiating \cite[Corollary~2]{koltchinskii2011nuclear} using our notations, we obtain an error bound
\begin{align}\label{bound2}
\frac{\|\mathbf{\widehat M} - \mathbf{M}^*\|_F^2}{nm}\leq c(\sigma^2+t)\frac{(n+m)d}{s}\log(n\vee m)
\end{align}
with high probability, where $c$ and $t$ are two constants related to $C_x$, $C_y$ and $C_m$. Note that 
our  error bound 
\begin{align}
O\Big(\big((|\mathcal{G}(\mathbf{X}^*)|+|\mathcal{G}(\mathbf{Y}^*)|)d+m+n\big)\log(n \vee m)/s\Big)\nonumber
\end{align}
can be much lower than~(\ref{bound2}) provided that  the total number of subgroups $|\mathcal{G}(\mathbf{X}^*)|+|\mathcal{G}(\mathbf{Y}^*)| \ll n+m$. However, to directly optimize~(\ref{esjky}) with the regularizer~(\ref{gxgy}) is undesirable in practice since the indicator function complicates the optimization and 
 the sets $\mathcal{X,Y}$ are discretized. Nonetheless, Theorem~\ref{th3} provides some theoretical justifications on the benefits of incorporating nonconvex pairwise penalties in reducing the estimation error.

\section{Numerical Experiments} \label{ex}
In this section, 
we first provide some heuristic rules to adaptively set the graph weights based on the partially observed data matrix in Section~\ref{ss:weight}. We then compare our algorithm $\bf{LLFMC}$ to the trace-norm regularized {\bf MC} without pairwise penalties~\cite{candes2009exact} and  the state-of-the-art GRMF algorithm, i.e., the graph regularized alternating Least Squares ({\bf GRALS})~\cite{rao2015collaborative} on different graphs and evaluate the effectiveness of the proposed adaptive weights in Section~\ref{ss:weight}. We use both synthetic and real datasets to verify the performance of the proposed framework.  
All experiments are performed using Matlab/Octave on a PC with macOS system, Intel 5 core 2.9 GHz CPU and 16G RAM. 

\subsection{Adaptive Weights Selection }\label{ss:weight}

Existing works~\cite{ma2011recommender,kalofolias2014matrix,rao2015collaborative,zhao2015expert} often construct the weight matrices $\mathbf{W}$ and $\mathbf{U}$ in \eqref{eq:genform4}
by incorporating side information, which may not always be available. Motivated by the analysis in Section~\ref{com_bound}, we introduce adaptive weights that are computed based on the partially observed matrix itself without any additional side information. Due to symmetry, we only explain how to adaptively select $w_{ij}$, and $u_{st}$ follows similarly. We consider the following weighted nearest neighbor rule for generating the weights $w_{ij}$, where
\begin{align}\label{wij}
w_{ij}=\mathbf{1}_{ij}^{k_w}/{\rm d}(\mathbf{x}_i,\mathbf{x}_j ),\quad i \neq j,
\end{align}
where ${\rm d}\left(\mathbf{x}_i,\mathbf{x}_j\right)$ approximates the similarity between $\mathbf{x}_i$ and $\mathbf{x}_j$,  and the indicator function $\mathbf{1}_{ij}^{k_w}$ is $1$ if $\mathbf{x}_j$ belongs to one of the $k_w$-nearest neighbors of $\mathbf{x}_i$ and $0$ otherwise. 
The rule~(\ref{wij}) makes $w_{ij}$ large during the computation to enforce $\mathbf{x}_i \approx \mathbf{x}_j$, similar to the adaptive weight chosen in~\cite{zou2006adaptive} for adaptive LASSO. 

 %Denote $\mathbf{ r}_i$ and $\mathbf {r}_j$ are the $i^{th}$ and $j^{th}$ rows of the partially observed matrix $\Psi_{\Omega}\left(\mathbf{M}\right)$ and the index set $\mathcal{S}_{i}$ as the index set of observed entries in the $i^{th}$ row. 

Below we provide two heuristic ways to approximate ${\rm d}\left(\mathbf{x}_i,\mathbf{x}_j\right)$ using $\Psi_{\Omega}(\mathbf{M})$. Recall $\Omega_i=\{ j:(i,j)\in\Omega\}$ is the index set of the observed entries in the $i^{th}$ row of $\Psi_{\Omega}(\mathbf{M})$. The first rule is based on the observation that the $i^{th}$ and $j^{th}$ rows of $\mathbf{M}$ should be close if $\mathbf{x}_i\approx \mathbf{x}_j$ when the noise is not too large. Therefore, we define
 \begin{align}\label{d1}
{\rm d_1}\left(\mathbf{x}_i,\mathbf{x}_j\right)=\left\{ \begin{array}{ll}
\left(\frac{\sum_{t\in \Omega_{i}\cap\Omega_{j}}\left(M_{it}- M_{jt}\right)^2}{\left|\Omega_{i}\cap\Omega_{j}\right|}\right)^{\frac{1}{2}} , &  \mathcal{S}_{i}\cap\mathcal{S}_{j}\neq \emptyset \\
0 & \mbox{otherwise}
\end{array}\right. .
\end{align}
 The key idea of ${\rm d_1}\left(\mathbf{x}_i,\mathbf{x}_j\right)$ in~\eqref{d1} is to use the ratings of user $\mathbf{x}_i$ and user $\mathbf{x}_j$ on the same items to measure their similarity. However, the above rule does not work very well when the observations are very sparse, since the number of  items simultaneously rated by both user $\mathbf{x}_i$ and  user $\mathbf{x}_j$ can be quite small. In that case, we use
\begin{align}\label{d2}
{\rm d_2}\left(\mathbf{x}_i,\mathbf{x}_j\right)=\left(\frac{\sum_{t\in\Omega_i \cup \Omega_j}\left( \bar{M}_{it} - \bar{M}_{jt} \right)^2}{\left|\Omega_i \cup \Omega_j\right|}\right)^{\frac{1}{2}},
\end{align}
where
\begin{equation}
\bar{M}_{it}  = \begin{cases}
 M_{it}  & \text{ if }  t\in\Omega_i,\\
\sum\limits_{r\in \mathcal{R}_t}M_{rt}/|\mathcal{R}_t|,\, &{\text{ if }}t\in\Omega_j \setminus \Omega_i,\hspace{1cm} \nonumber
\end{cases}
\end{equation}
where  $\mathcal{R}_t$ is the set of indices of the observed entries for the $t^{th}$ column of $\mathbf{M}$. In \eqref{d2}, for the items rated only by user $\mathbf{x}_i$, we estimate the ratings of user $\mathbf{x}_j$ on these items by taking the average of the ratings on these items. This step is motivated by the observation that the ratings of each user are not very far away from their average. 
Using ${\rm d_2}\left(\mathbf{x}_i,\mathbf{x}_j\right)$, we obtain a solution $\mathbf{X}^{(1)}$ and $\mathbf{Y}^{(1)}$ of the optimization problem~(\ref{eq:genform4}) with weights computed by~(\ref{wij}). Next, we use $\mathbf{X}^{(1)}$ and $\mathbf{Y}^{(1)}$ to compute the distance by ${\rm d}\left(\mathbf{x}_i,\mathbf{x}_j\right)= \|\mathbf{x}^{(1)}_i-\mathbf{x}^{(1)}_j \|_2$. The performance of these heuristic rules will be examined in the experiments below.

\subsection{Synthetic Data with Subgroups}\label{reducedError} 
To verify Theorem~\ref{th3}, we generate a ground truth matrix $\mathbf{M}^*=\mathbf{X}^*\mathbf{Y}^{*T}\in\mathbb{R}^{n\times n}$, $\mathbf{X}^*\in\mathbb{R}^{n\times d}$, $\mathbf{Y}^*\in\mathbb{R}^{n\times d}$ in the following way. Let  $\mathbf{U}_{1}$ be composed of entries drawn uniformly at random from $[0,1]^d$ and $\{\mathbf{U}_{i+1}-\mathbf{U}_i,1\leq i<k_x\}$ be $\text{i.i.d.}$ and uniformly distributed on $[0,10]^d$. Similarly we generate $\mathbf{V}_{1}, \cdots,\mathbf{V}_{k_y}$. We evenly divide the feature vectors $\{\mathbf{x}^*_i\}$ into $k_x$  subgroups and set the features in the $i^{th}$ subgroup to be $\mathbf{U}_{i}$. We divide $\{\mathbf{y}^*_j\}$ in a similar way. When $k_x=k_y=n$, there exist no subgroups. Normalizing by a constant, we scale $\mathbf{M}^*$ to satisfy $\|\mathbf{M}^*\|_F=10^6$.
The observed entries of $\Psi_{\Omega}(\mathbf{M})$ are sampled from $\mathbf{M=M^*+N}$ uniformly at random with a sample rate $0<\rho<1$, with the noise $\mathbf{N}$ containing $\text{i.i.d.}$ entries following $\mathcal{N}(0,\sigma^2)$.  
Moreover, we use the rule~(\ref{d1}) to compute the graph weights by setting $k_w=0.15n$. We fix $n=200$ and $d=5$.
Let $\mathbf{\widehat{M}}$ be the estimate of $\mathbf{M}^*$. 
 
 \smallskip
\noindent{\bf Relative error.} We study three cases where the number of subgroups is small ($k_x=k_y=20$), medium ($k_x=k_y=50$) and large ($k_x=k_y=200$).  
 Fig.~\ref{errorcurve} shows that our method LLMFC with the MCP significantly reduces the recovery error compared to standard MC, even when no real subgroups exist. 
 \begin{figure}[ht]
\centering
\includegraphics[width=0.5\textwidth]{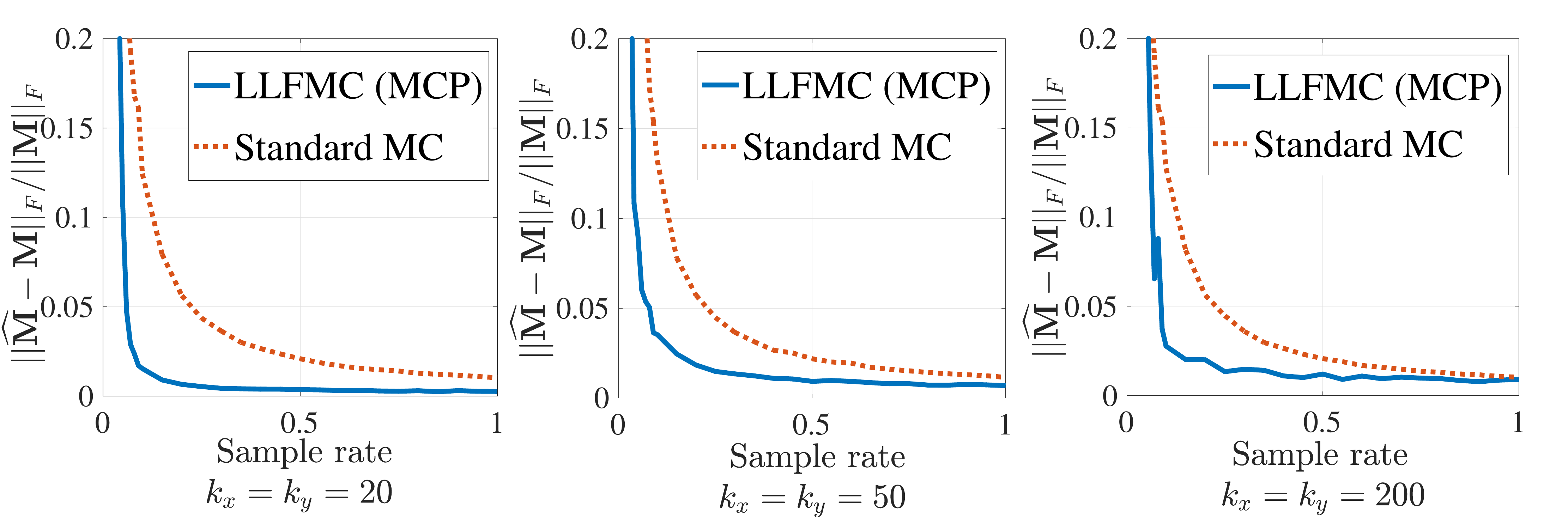} 
\caption{The relative error {w.r.t.}  sample rate: $n=200$, $ \sigma=10^{2}.$ }\label{errorcurve}
\end{figure} 

\noindent{\bf Subgroup identification.} Similar to~\cite{han2015learning}, we identify the subgroups for the recovered features $\{\mathbf{\widehat x}_i\}$ and $\{\mathbf{\widehat y}_j\}$ 
through the matrices $\mathbf{S}^x$, $\mathbf{S}^y$, where the features $\mathbf{\widehat x}_{u}$ and $\mathbf{\widehat x}_{v}$ (respectively $\mathbf{\widehat y}_{s}$ and $\mathbf{\widehat y}_{t}$) are in the same subgroup if and only if $S_{uv}^x=1$ (respectively $S_{st}^y=1$). We set $S_{uv}^x=1$ if $\|\mathbf{\widehat x}_{u}-\mathbf{\widehat x}_{v}\|_2<0.01\min\{\|\mathbf{\widehat x}_{u}\|_2,\|\mathbf{\widehat x}_{v}\|_2\}$ and~$0$ otherwise. Similarly, we define the matrix $\mathbf{S}^y$. Due to the limited space, we only visualize the matrix $\mathbf{S}^x$ in Fig.~\ref{subgroup}. The results clearly show that our method recovers the subgroup structure quite well. In contrast,  the standard MC cannot identify this latent structure.    
 
\begin{figure}[h]
\centering
\includegraphics[width=0.5\textwidth]{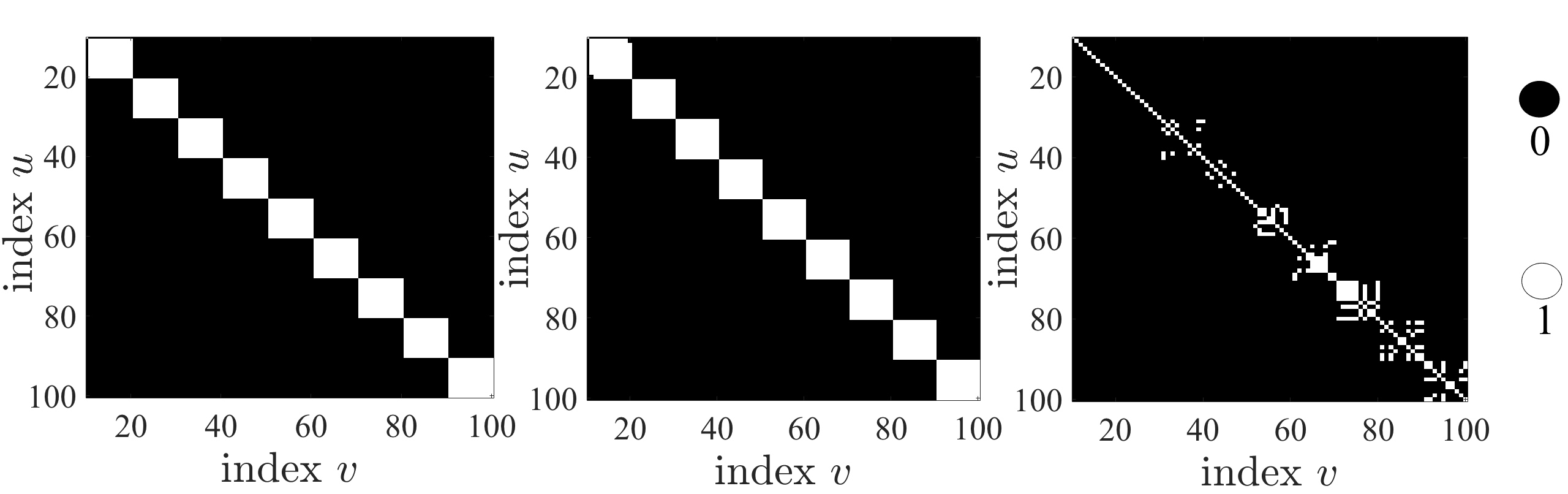} 
\caption{Illustration of $\mathbf{S}^x$ obtained by the ground truth (left), LLFMC (MCP) (middle) and the standard MC (right): $n=100$, $\sigma=100$, $k_x=k_y=10$ and the sample rate $\rho=0.3$. }\label{subgroup}
\end{figure}

\subsection{Comparison to GRALS on Real Data} \label{noncon}

We conduct the experiments on two groups of real collaborate filtering datasets: Jester\footnote{\url{http://goldberg.berkeley.edu/jester-data}}~\cite{goldberg2001eigentaste} and MovieLens\footnote{\url{http://grouplens.org/datasets/movielens}}~\cite{harper2016movielens}. Their statistics are presented in Table~\ref{t102}. The Jester datasets contain anonymous ratings of 100 jokes by users, where the ratings are real values ranging from $-10.00$ to $10.00$. 
For MovieLen100k and MotiveLens1M datasets, the ratings of the movies have 5 scores ($1$-$5$). 
 
\begin{table}[ht]
%\small
\begin{center}
\caption{Summary statistics of datasets} 
%\vspace{0.1cm}
\scalebox{0.97}{
\begin{tabular}{| l |c|c |c|}
\hline
{Datasets} & \# of users ($n$) & \# of items ($m$)& Density\\ \hline \hline 
 Jester1 & 24983 & 100 & 0.7247 \\ 
 Jester2 & 23500 & 100 & 0.7272 \\ 
 Jester3 & 24938 & 100 & 0.2474 \\ 
 MovieLens100K& 943 & 1682 & 0.0630 \\ 
 MovieLens1M & 6040 & 3706 & 0.0447 \\ 
 \hline 
\end{tabular}\label{t102}
}
\end{center}
\end{table}

Following~\cite{rao2015collaborative}, we evaluate MovieLens100K using the five provided data splits. For MovieLens1M, we randomly split the observed ratings into a training set ($90\%$) and a test set ($10\%$).
For Jester datasets, we randomly choose ${10\%}$ of the observed ratings as the training set and the remaining ${90\%}$ as the test set. We use the root mean squared error ($\rm{RMSE}$) to measure the prediction performance, given by
${\rm RMSE}=\big(\sum_{i=1}^N\left(\widehat r_{i}-r_{i}\right)^2/N\big)^{1/2}$,
where $r_i$ is the observed rating in the test set, and $\widehat r_i$ is the predicted rating and $N$ is the total number of ratings in the test set. When reporting the performance, the best performance in each scenario is bolded.

For our algorithm LLFMC, we choose the regularization parameters $\gamma_X$ and $\gamma_Y$ from $2^{\{-2,-1,...,10\}}$, fix $\alpha=1$ and set $\eta=10^4$ to guarantee the convergence.
In addition, we fix $d=4$ for MovieLens datasets and $d=100$ for Jester datasets. 
For MCP~(\ref{mcp}), we select $t$ from $\{0.5,2,20\}$, and for the M-type penalty function, we fix $b=3$. 
We downloaded the  code of {\bf GRALS} from the author's website\footnote{\url{http://nikrao.github.io/Code.html}}, and changed the number of iterations from $10$ (default) to $100$ to ensure its convergence. 

\smallskip
\noindent{\bf Comparison with GRALS with side information.} We first use the same graph constructed in~\cite{rao2015collaborative} based on side information obtained therein. Table~\ref{t3} shows that our method still makes improvement over GRALS using the same graphs as in~\cite{rao2015collaborative}, which might be due to the bias reduction effect of using non-convex penalties. 
\begin{table}[h]

\begin{center} 
\caption{RMSE on MovieLens100K, using the same graphs as in~\cite{rao2015collaborative} that were constructed by side information.}
\scalebox{1}{
\begin{tabular}{ c|c|c |c|c}
\hline
{\bf Methods}  &MC &GRALS  & LLFMC (M-type) & LLFMC (MCP) \\ \hline\hline
 RMSE &$0.973$&$0.945$ & $0.930$& $ {\bf 0.927}$\\ 
 \hline
\end{tabular}\label{t3}
}
\end{center} 
\end{table}
 
\smallskip
\noindent{\bf Adaptive weights.} We test the performance of different algorithms under different adaptive schemes of weight selection. The first one is an unweighted 100-nearest neighbor graph, where the nearest neighbors are determined using the distance~(\ref{d2}), and the edges all weighted by $1$. Table~\ref{t2} summarizes the performance of different methods using an unweighted 100-nearest neighbor graph for both users and movies. It shows that with MCP and M-type penalties, our method outperforms GRALS in all datasets. 
\begin{table}[h] 
\small
\begin{center}
\captionsetup{justification=centering}
\caption{RMSE on real datasets using unweighted 100-nearest neighbor graphs.} 
\scalebox{0.95}{
\begin{tabular}{ l |c|c |c}
\hline
Datasets  & GRALS & LLFMC (MCP) & LLFMC (M-type) \\ \hline\hline
 Jester1 &$4.801 $ &$ 4.712 $& ${\bf 4.710}$\\ 
 Jester2 & $4.880$ & ${\bf 4.831}$& $4.838$\\ 
 Jester3 &$6.701$ & $6.651$ & ${\bf 6.634}$\\ 
 MovieLens100K & $0.950$ &${\bf 0.934}$& $0.936$ \\ 
 MovieLens1M  & $0.873$ &  ${\bf 0.859}$ &$0.861$\\ 
 \hline
\end{tabular}\label{t2}
}
\end{center}
\end{table}

We next examine the performance using the graph constructed based on our adaptive weights defined in~\eqref{wij} using the distance~\eqref{d2}.  
We set $k_w=k_u=10$ in our wights~(\ref{wij}) for MovieLens datasets and set $k_w=20$, $k_u=2$ for Jester datasets. Table~\ref{t4} summarizes the performance of different methods using the adaptive weights. Compared with Table~\ref{t2} and  Table~\ref{t3}, it can be seen that the proposed adaptive weights significantly improve both the performance of GRALS and our methods on all datasets.
In addition, our method still outperforms the competitors in almost all experiments. 
\begin{table}[!t]
\begin{center} 
\caption{RMSE on real datasets using the graph based on the adaptive weights~(\ref{wij}) using the distance~(\ref{d2}).}
\scalebox{0.95}{
\begin{tabular}{ l |c|c |c}
\hline
Datasets  & GRALS & LLFMC (MCP) & LLFMC (M-type) \\ \hline \hline
 Jester1 &$4.713$  & ${\bf 4.651}$ & $4.672$\\ 
 Jester2 & ${\bf 4.789}$ & $4.801$ & $4.792$\\ 
 Jester3 &${\bf 6.461}$ & $6.551$ &${\bf 6.461}$\\ 
 MovieLens100K & $0.919$ & ${\bf 0.909}$ & $0.914$ \\ 
 MovieLens1M  & $0.851$ &  ${\bf 0.847}$ &$0.850$\\ 
 \hline
\end{tabular}\label{t4}
}
\end{center} 
\end{table}

\smallskip
\noindent{\bf Subgroup identification.} 
We illustrate the subgroups of users and movies from the MovieLens100K dataset computed by LLFMC (MCP) following the similar manner of Fig.~\ref{subgroup}. Fig.~\ref{sub1} clearly shows 
that the users and movies exhibit subgroup structures; only $200$ users and $200$ movies are plotted for a better presentation. 
To get more useful information, we investigate the $2^{th}$ largest user subgroup in Fig.~\ref{sub3}. Interestingly, we find out that most users in this subgroup have ages between $25$ and $35$, consisting of mainly engineers and students.  
\begin{figure}[ht]
\centering
\vspace{-0.1in}
\includegraphics[width=0.5\textwidth]{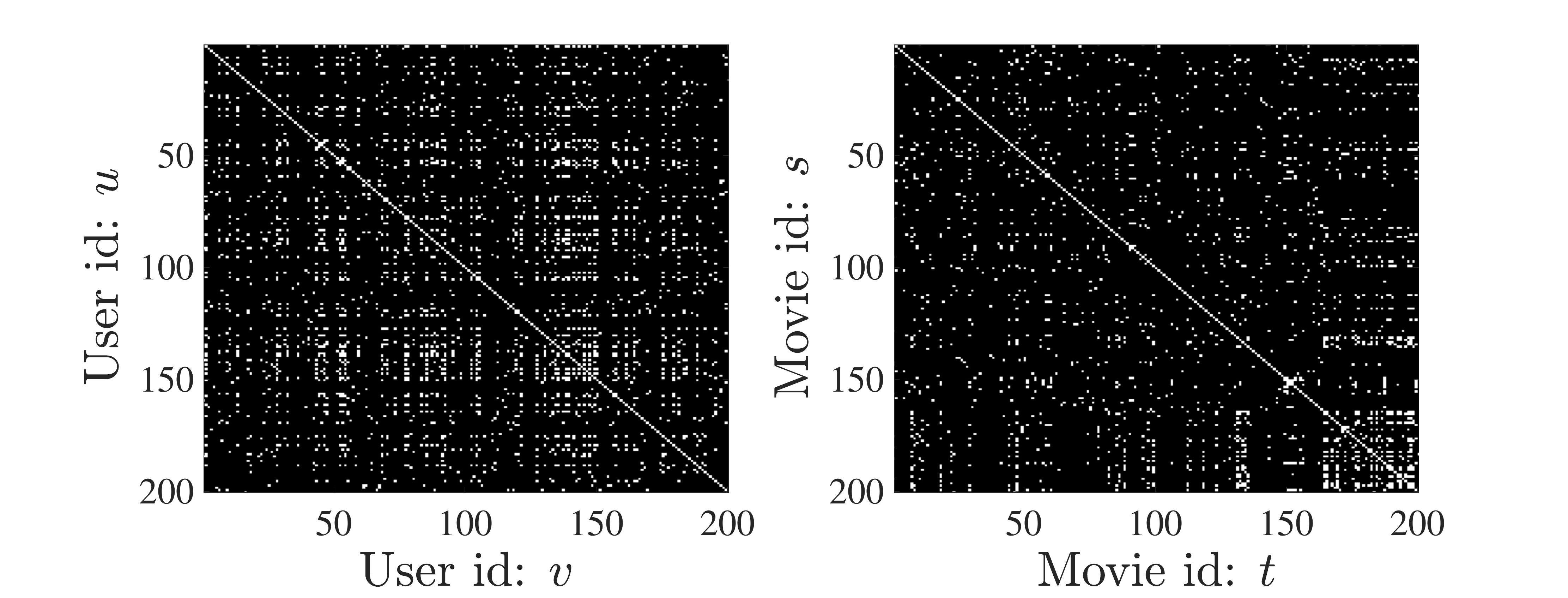} 
\caption{Illustration of $\mathbf{S}^x$ (left) and $\mathbf{S}^y$ (right): white spots imply that the users/items are in the same subgroup. } \label{sub1}
\end{figure}

\begin{figure}[ht!]
\centering
\includegraphics[width=0.5\textwidth]{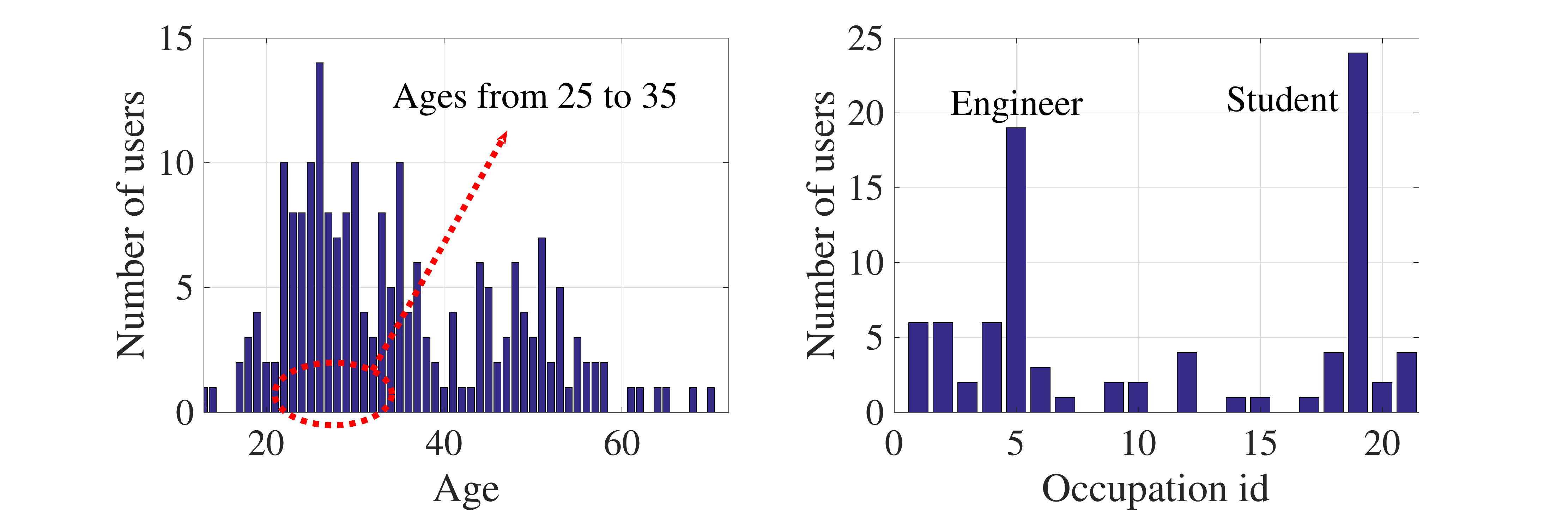} 
\caption{Further insights into the $2^{th}$ largest user subgroup. Left: the age distribution.  Right: the occupation distribution for the users with age between $25$ and $35$ in the same subgroup.} \label{sub3}
\end{figure}

\section{Conclusions} \label{sec:conclusions}
 
We propose a new optimization framework to learn latent features in low-rank matrix completion by a wide class of (non-)convex pairwise penalty functions.
To efficiently solve this class of optimization problems, we develop an efficient algorithm and prove its convergence guarantee. 
On the statistical side, we characterize the complexity-regularized maximum likelihood estimator under the subgroup-based model, which sheds light to the statistical benefit of incorporating pairwise penalty functions. Extensive experiments on both synthetic and real datasets confirm the superior performance of our framework.

\bibliography{refs,ref2}
\bibliographystyle{IEEEtran}

 \newpage

\appendix \label{sec:appendix}
 
\subsection{Proof of Proposition~\ref{nnzE}}\label{apen:A} 
Let $\mathbf{e}_i$ be the $i^{th}$ row vector of the matrix $\mathbf{E}_x$.  Recalling the definition of the matrix $\mathbf{E}_x$ given in Section~\ref{compute}, we  make the following three key observations for $\mathbf{E}_x$.
\begin{enumerate}
\item Each $\mathbf{e}_i$ has exactly $k_w$ non zeros, $1\leq i \leq n$. \label{i1}
\item For  any two row vectors $\mathbf{e}_i$ and $\mathbf{e}_j$, there exists at most one coordinate  at which $\mathbf{e}_i$ and $\mathbf{e}_j$ are both non-zero.\label{i2}
%have at most $1$ common non-zero index, i.e., there is at most one element in the set $\{t: \mathbf{E}_i(t)\neq\;  \text{and} \; \mathbf{E}_j(t)\neq 0\}$.
\item Each column vector of the matrix $\mathbf{E}_x$ has exactly two non zeros. \label{i3}
\end{enumerate}
Based on these observations, we claim that each row vector of $\mathbf{E}_x\mathbf{E}_x^T$, i.e., $\mathbf{e}_i\mathbf{E}_x^T=[\mathbf{e}_i\mathbf{e}_1^T,\mathbf{e}_i\mathbf{e}_2^T,...,\mathbf{e}_i\mathbf{e}_n^T]$ has at most $k_w+1$ non zeros for any $1\leq i\leq n$. We prove this result by contradiction. 
%Without loss of generality, we consider the scenario where $i=1$.  Let $\{d_s,1\leq s\leq k_w\}$ be the coordinates at which $\mathbf{E}_1$ is non-zero. 
%To simplify the notation, we define $\mathbf{v}_j=$
Suppose there exists a  vector $\mathbf{e}_s\mathbf{E}_x^T=[\mathbf{e}_s\mathbf{e}_1^T,\ldots,\mathbf{e}_s\mathbf{e}_n^T]$ that  has  $k_w+2$ nonzeros. Without loss of generality, we assume  that  $s=1$ and  $\mathbf{e}_1\mathbf{e}_t^T\neq 0$ for all $1\leq t\leq k_w+2$. 
Let $\mathcal{C}_1=\{c_u: \mathbf{e}_1(c_u)\neq 0, 1\leq u\leq k_w\}$ be a set of coordinates at which $\mathbf{e}_1$ is non-zero. 
Based on the observation~\ref{i2}), we have, for all $\mathbf{e}_t, 2\leq t\leq k_w+2$,  there exists exactly one coordinate in $\{d_s,1\leq s\leq k_w\}$ at which $\mathbf{e}_t$ is non-zero. Thus, there must exist one coordinate $c_v$ in $\mathcal{C}_1$ such that  $\mathbf{e}_1(c_v)\neq 0$, $\mathbf{e}_p(c_v)\neq 0$ and $\mathbf{e}_q(c_v)\neq 0$, for certain $2\leq p,q\leq k_w+2$. However, this contradicts the observation~\ref{i3}). By contradiction, we can conclude that each vector $\mathbf{e}_i\mathbf{E}_x^T$ has at most $k_w+1$ non zeros. Then, the proof is completed by summing over all rows of $\mathbf{E}_x\mathbf{E}_x^T$. 

\subsection{Proof of Theorem~\ref{th:exact}}\label{apen:C}

We first provide a useful property of our algorithm.
\begin{prop}[{\bf Full column rank}]\label{graph}
After Step 1 of the LLFMC algorithm, the matrices $\mathbf{E}_x$ and $\mathbf{E}_y$ defined in~(\ref{eq:constraint2}) are both full column rank.  Thus, there exist positive constants $\sigma_x,\sigma_y$ such that $\mathbf{E}_x^T\mathbf{E}_x\succeq \sigma_x \mathbf{I}$ and $\mathbf{E}_y^T\mathbf{E}_y\succeq \sigma_y \mathbf{I}.$  
\end{prop}
 \begin{proof}
  Recall that after Step 1, the graph becomes acyclic, where $\mathbf{E}_x$ is the incidence matrix associated to this graph. By~\cite[Lemma~2.5]{bapat2010graphs}, we have that all columns $\mathbf{e}_1,\mathbf{e}_2,\cdots,\mathbf{e}_{\left|\varepsilon_{X}\right|}$ of $\mathbf{E}_x$ are linearly independent and hence $\mathbf{E}_x$ is full column rank. Similar result holds for $\mathbf{E}_y$. Then, using Remark 1 in~\cite{li2015global}, we finish the proof.
  \end{proof}
  
  This proposition immediately leads to the following lemma.
  \begin{lemma}\label{PD}
For any $ \mathbf{A}_x\in\mathbb{R}^{d\times\left|\varepsilon_{X}\right|}$ and $\mathbf{A}_y\in\mathbb{R}^{d\times\left|\varepsilon_{Y}\right|}$, we have $\|\mathbf{A}_x\mathbf{E}_x^T\|_F^2\geq\sigma_x\|\mathbf{A}_x\|_F^2$ and $\|\mathbf{A}_y\mathbf{E}_y^T\|_F^2\geq\sigma_y\|\mathbf{A}_y\|_F^2$.
\end{lemma}
\begin{proof}
Recalling that $\mathbf{E}_x^T\mathbf{E}_x-\sigma_x \mathbf{I}\succeq0$, we have
 $\|\mathbf{AE}_x^T\|_F^2-\sigma_x\|\mathbf{A}\|_F^2=\tr\left[\mathbf{A}_x(\mathbf{E}_x^T\mathbf{E}_x-\sigma_x \mathbf{I})\mathbf{A}_x^T\right]=\sum_{i=1}^{d}\mathbf{A}_x^i(\mathbf{E}_x^T\mathbf{E}_x-\sigma_x \mathbf{I})(\mathbf{A}_x^i)^T\geq0,$ where $\mathbf{A}_x^i$ is the $i^{th}$ row vector of the matrix $\mathbf{A}_x$. Similarly, we can prove $\|\mathbf{A}_y\mathbf{E}_y^T\|_F^2\geq\sigma_y\|\mathbf{A}_y\|_F^2$.
\end{proof}

We next establish the following Lemmas~\ref{le:lambda}, \ref{le:pq} and \ref{le:xyz} to lower bound the descent of $\mathcal{L}_{\eta}$ during each subproblem.  
 
\begin{lemma}[{\bf Upper-bound the descents of multipliers}]\label{le:lambda}
Suppose that Assumption~\ref{bound} holds .  Then, 
For each $k\in\mathbb{N}^{+}$, there exists a constant $L_y>0$ such that 
\begin{align}\label{le:lam}
\|\mathbf{\Lambda}^{k+1}-&\mathbf{\Lambda}^k\|_F^2\nonumber
\\ \leq&\frac{2}{\sigma_x}\|\mathbf{X}^k-\mathbf{X}^{k-1}\|_F^2+\frac{2(M_x+1)L_y}{\sigma_x}\|\mathbf{Y}^k-\mathbf{Y}^{k-1}\|_F^2\nonumber
\\&+\frac{2L_y+2(\alpha+1)^2}{\sigma_x}\|\mathbf{X}^{k+1}-\mathbf{X}^k\|_F^2.
\end{align}
\end{lemma}
\begin{proof}
Noting that $\text{vec}((\mathbf{X}^{k+1})^T)$ is the solution to
\begin{align}
(\mathbf{G}^k_y+&(\eta\mathbf{E}_x\mathbf{E}_x^T+(\alpha+1))\otimes \mathbf{I}_d)\text{vec}((\mathbf{X}^{k+1})^T)  \nonumber
\\&= (\mathbf{b}^k_y)^T+\text{vec}((\mathbf{X}^k)^T+\eta\mathbf{P}^{k+1}\mathbf{E}_x^T+\mathbf{\Lambda}^k\mathbf{E}_x^T),
\end{align} 
which, together with $\mathbf{\Lambda}^{k+1}-\mathbf{\Lambda}^{k}=\eta(\mathbf{ P}^{k+1}-(\mathbf{X}^{k+1})^T\mathbf{E}_x)$ and the identity $\text{vec}(\mathbf{AB})=(\mathbf{B}^T\otimes \mathbf{I})\text{vec}(\mathbf{A})$, implies
%\begin{small}
\begin{align}\label{eq:Lam}
\text{vec}(\mathbf{\Lambda}^{k+1}\mathbf{E}_x^T)
=&(\mathbf{G}_y^k+(\alpha+1)\mathbf{I}_{nd})\text{vec}((\mathbf{X}^{k+1})^T) \nonumber
\\&-(\mathbf{b}_y^k)^T-\text{vec}((\mathbf{X}^{k})^T).
\end{align}
%\end{small}
Using~(\ref{eq:Lam}),  we further obtain
\begin{align}\label{eq:lambda}
\|(\mathbf{\Lambda}^{k+1}&-\mathbf{\Lambda}^k)\mathbf{E}_x^T\|_F^2 =\|\text{vec}((\mathbf{\Lambda}^{k+1}-\mathbf{\Lambda}^k)\mathbf{E}_x^T) \|_2^2\nonumber
\\=&\|(\mathbf{G}_y^k-\mathbf{G}_y^{k-1})\text{vec}((\mathbf{X}^{k+1})^T)+(\mathbf{b}_y^{k-1}-\mathbf{b}_y^k)^T \nonumber
\\&+\mathbf{G}_y^{k-1}\text{vec}((\mathbf{X}^{k+1}-\mathbf{X}^k)^T)+\text{vec}((\mathbf{X}^{k-1}-\mathbf{X}^{k})^T) \nonumber
\\&+(\alpha+1)\text{vec}((\mathbf{X}^{k+1}-\mathbf{X}^k)^T)\|_2^2\nonumber
\\\leq&\;2M_x\|\mathbf{G}_y^k-\mathbf{G}_y^{k-1}\|_F^2+2\|\mathbf{X}^{k-1}-\mathbf{X}^{k}\|_F^2\nonumber
\\&+2(\alpha+1)^2\|\mathbf{X}^{k+1}-\mathbf{X}^k\|_F^2+2\|\mathbf{b}_y^{k}-\mathbf{b}_y^{k-1}\|_2^2\nonumber
\\&+2\|\mathbf{G}_y^{k-1}\|_F^2\|\mathbf{X}^{k+1}-\mathbf{X}^k\|_F^2,
\end{align}
where the inequality follows from the facts that $\|\mathbf{a+b}\|_2^2\leq(\|\mathbf{a}\|_2+\|\mathbf{b}\|_2)^2\leq2(\|\mathbf{a}\|_2^2+\|\mathbf{b}\|_2^2)$, $\|\mathbf{Aa}\|_2\leq \|\mathbf{A}\|_F\|\mathbf{a}\|_2$ and $\|\mathbf{ab}\|_2\leq\|\mathbf{a}\|_2\|\mathbf{b}\|_2$. 
Recalling the definition of $\mathbf{G}_y^k$ and $\mathbf{b}_y^k$ in~(\ref{cg}) and using Assumption~\ref{bound}, there exists a sufficiently large constant $L_y$ such that
\begin{align} 
&\|\mathbf{G}_y^k\|_F^2\leq L_y, \|\mathbf{G}_y^k-\mathbf{G}_y^{k-1}\|_F^2\leq L_y\|\mathbf{Y}^k-\mathbf{Y}^{k-1}\|_F^2\nonumber
\\&\|\mathbf{b}_y^k\|_2^2\leq L_y, \|\mathbf{b}_y^{k}-\mathbf{b}_y^{k-1}\|_2^2\leq L_y\|\mathbf{Y}^k-\mathbf{Y}^{k-1}\|_F^2,\nonumber
\end{align}
which, by~(\ref{eq:lambda}) and Lemma~\ref{PD}, yields the proof. 
\end{proof}
Using a similar approach as in Lemma~\ref{le:lambda}, we have
\begin{align}\label{le:V}
&\|\mathbf{V}^{k+1}-\mathbf{V}^k\|_F^2   \nonumber
\\&\leq \frac{2}{\sigma_y}\|\mathbf{Y}^k-\mathbf{Y}^{k-1}\|_F^2+\frac{2(M_y+1)L_x}{\sigma_y}\|\mathbf{X}^{k+1}-\mathbf{X}^{k}\|_F^2\nonumber
\\&+\frac{2L_x+2(\alpha+1)^2}{\sigma_y}\|\mathbf{Y}^{k+1}-\mathbf{Y}^k\|_F^2.
\end{align}
%The following two lemmas estimate the decrease of $\mathcal{L}_{\eta,\tau}\left(\bar P,\bar Q,Z,X,Y,\Lambda,V\right)$ at each iteration.

%Define a constant $\eta_0=2\varsigma_0\max_{i \in \varepsilon_X, j\in \varepsilon_Y}(w_i, u_j)+1$ .
%%with $\widehat w=\max_{l \in \varepsilon_X}\{w_l\}$ and $\widehat u=\max_{j\in \varepsilon_Y }\{u_{j}\}$.  
%Then we have the lemma below.
\begin{lemma}[{\bf Lower bound descents of $\mathcal{L}_\eta$ for $\mathbf{P}$, $\mathbf{Q}$ subproblems}]\label{le:pq}  
	Let $\eta_0$ be a constant satisfying $\eta_0> 2\varsigma_0\max_{i \in \varepsilon_X}(w_i)$ and $\eta_0> 2\varsigma_0 \max_{j\in \varepsilon_Y}(u_j)$, where $\varepsilon_X$, $\varepsilon_Y, w_i, u_j$ are defined in \eqref{model}. Then, for any $\eta>\eta_0$ and each $k\in 
\mathbb{N}^+$, we have    
%\begin{small}
\begin{align}
\mathcal{L}_{\eta}& (\mathbf{ P }^{k+1},\mathbf{ Q }^{k},\mathbf{X}^k,\mathbf{Y}^k,\mathbf{\Lambda}^k,\mathbf{V}^k )-\mathcal{L}_{\eta} (\mathbf{ P }^{k},\mathbf{ Q}^k,
\mathbf{X}^k,\mathbf{Y}^k, \mathbf{\Lambda}^k,\mathbf{V}^k )\nonumber
\\&\leq -\frac{(\eta-\eta_0)}{2}\|\mathbf{ P }^{k+1}-\mathbf{ P }^k\|_F^2.\nonumber
\end{align} 
%\end{small}
\begin{proof}
Define the function $\widetilde p(\mathbf{p}_{l},\lambda_X)=w_{l}\hspace{0.04cm}p\left(\|\mathbf{p}_{l}\|_2,\gamma_X\right)+\eta_0\|\mathbf{p}_{l}-(\mathbf{X}^k)^T\mathbf{E}_x[l] \|_{2}^{2}/2$. Recalling Assumption~\ref{c:1} and noting $\eta_0>2\widehat w\varsigma_0$, we have that the function $\widetilde p$ is strongly convex in $\mathbf{p}_{l}$. Based on the optimality of $\mathbf{p}_{l}^{k+1}$ for~(\ref{pl}), we have 
\begin{align}
\mathbf{0}\in\partial_{\mathbf{p}_{l}}\, \widetilde p\left(\mathbf{p}_{l}^{k+1},\gamma_X\right)+\mathbf{\Lambda}^{k}[l]+(\eta-\eta_0)(\mathbf{p}_{l}^{k+1}-(\mathbf{X}^k)^T\mathbf{E}_x[l]).
\end{align}
Let $\mathbf{ d}_k:=-\mathbf{\Lambda}^{k}[l]-(\eta-\eta_0)(\mathbf{p}_{l}^{k+1}-(\mathbf{X}^k)^T\mathbf{E}_x[l])$, which is a subgradient of the convex function $\widetilde p$ at $\mathbf{p}_{l}=\mathbf{p}_{l}^{k+1}.$ By the definition of the subgradient of a convex function, we have
\begin{align}\label{eq:subgradient}
\widetilde p(\mathbf{p}_{l}^{k},\gamma_X)-\widetilde p(\mathbf{p}_{l}^{k+1},\gamma_X)\geq\langle \mathbf{ d}_k,\mathbf{p}_{l}^{k}-\mathbf{p}_{l}^{k+1}\rangle.
\end{align}
Let $F_{l}(\mathbf{p}_{l}) =\widetilde p\left(\mathbf{p}_{l},\gamma_X\right)+(\mathbf{p}_{l}-(\mathbf{X}^k)^T\mathbf{E}_x[l])\mathbf{\Lambda}^k[l]+(\eta-\eta_0)\|\mathbf{p}_{l}-(\mathbf{X}^k)^T\mathbf{E}_x[l]\|_2^2/2$. Using~(\ref{eq:subgradient}), we have 
\begin{align}\label{le:each}
&F_{l}(\mathbf{p}_{l}^{k+1}) -F_{l}(\mathbf{p}_{l}^k) \nonumber
\\&=\widetilde p(\mathbf{p}_{l}^{k+1},\gamma_X)-\widetilde p(\mathbf{p}_{i}^{k},\gamma_X)+ (\mathbf{p}_{l}^{k+1}-\mathbf{p}_{l}^{k})\mathbf{\Lambda}^k[l] \nonumber
\\&-\frac{\eta-\eta_0}{2}\left(\|\mathbf{p}_{l}^k-(\mathbf{X}^k)^T\mathbf{E}_x[l]\|_2^2-\|\mathbf{p}_{l}^{k+1}-(\mathbf{X}^k)^T\mathbf{E}_x[l]\|_2^2\right)\nonumber
\\&\overset{(i)}=\,\widetilde p(\mathbf{p}_{l}^{k+1},\gamma_X)-\widetilde p(\mathbf{p}_{l}^{k},\gamma_X)-\langle\mathbf{d}_k,\mathbf{p}_{l}^{k+1}-\mathbf{p}_{l}^{k}\rangle\nonumber
\\&-\frac{\eta-\eta_0}{2}\|\mathbf{p}_{l}^{k+1}-\mathbf{p}_{l}^{k}\|_2^2
\leq-\frac{\eta-\eta_0}{2}\|\mathbf{p}_{l}^{k+1}-\mathbf{p}_{l}^{k}\|_2^2,
\end{align}
where $(i)$ follows from the fact that  $\|\mathbf{a+c}\|_2^2-\|\mathbf{b+c}\|_2^2=\|\mathbf{a-b}\|_2^2+2\langle \mathbf{b+c,a-b}\rangle$. 
Using~(\ref{le:each}), we have
\begin{small}
\begin{align}
\mathcal{L}_{\eta}&\left(\mathbf{ P}^{k+1},\mathbf{ Q}^{k},\mathbf{X}^k,\mathbf{Y}^k,\mathbf{\Lambda}^k,\mathbf{V}^k\right)-\mathcal{L}_{\eta}\left(\mathbf{ P}^{k},\mathbf{ Q}^k,\mathbf{X}^k,\mathbf{Y}^k,\mathbf{\Lambda}^k,\mathbf{V}^k\right)\nonumber
\\=&\sum_{l\in  \varepsilon_{X}}\left(F_{l}(\mathbf{p}_{l}^{k+1})-F_{l}(\mathbf{p}_{l}^k)\right)\leq-\frac{\eta-\eta_0}{2}\|\mathbf{ P}^{k+1}-\mathbf{ P}^{k}\|_F^2,\nonumber
%\\\leq& -\frac{\eta-\eta_x}{2}\sum_{i=1}^{\left|\varepsilon_{X}\right|}\left(\|\bar P_{i}^{k+1}-\bar P_{i}^{k}\|_2^2\right)\nonumber 
\end{align}
\end{small}
\hspace{-0.12cm}which finishes the proof.
%Although the MCP function $p\left(\bar P_{i},\gamma_X\right)$ is concave, we can extend the results of part (i) to part(ii) by making the following modifications. 
%{\bf Part (ii)}:
%Define a function $\widetilde p(\bar P_{i})=w_{l^{i}}p\left(\bar P_{i},\gamma_X\right)+\left(w_{max}/t+\epsilon\right)\|\bar P_{i}-X^kE_x^{i}\|_2^2/2$. By the results in~\cite{Ma2016}, $\widetilde p(\bar P_{i})$ is convex since $w_{max}/t+\epsilon>w_{l^{i}}/t$ for $1\leq i\leq \left|\varepsilon_{X}\right|$. Then, the results of part (i) can be extended to part (ii) by replacing $w_{l^{i}}p(\bar P_{i},\gamma_X)$ and $\eta$ in part (i) by $\widetilde p(\bar P_{i})$ and 
%$\eta-w_{max}/t-\epsilon$, respectively.
\end{proof}
\end{lemma}
By the symmetry of $\mathbf{ P}$ and $\mathbf{ Q}$, we derive similar results for $\mathbf{ Q}$ by replacing $\mathbf{ P}$ in Lemma~\ref{le:pq} by $\mathbf{ Q}$. 
%and $\widehat v$ respectively, where $\widehat v=\max_{1\leq j\leq \left|\varepsilon_{Y}\right|}\{v_{l^{j}}\}$.
\begin{lemma}[{\bf Upper bound the descents of $\mathcal{L}_\eta$ for $\mathbf{X,Y}$ subproblems}]\label{le:xyz}
For $\forall\, k\in \mathbb{N}^+$, we have 
\begin{align}
\mathcal{L}_{\eta}(\mathbf{ P}^{k+1},&\mathbf{ Q}^{k+1},\mathbf{X}^{k+1},\mathbf{Y}^{k+1},\mathbf{\Lambda}^k,\mathbf{V}^k) \nonumber
\\&-\mathcal{L}_{\eta}(\mathbf{ P}^{k+1},\mathbf{ Q}^{k+1},\mathbf{X}^k,\mathbf{Y}^k,\mathbf{\Lambda}^k,\mathbf{V}^k)\nonumber
\\\leq&-\frac{1}{2}\left(\|\mathbf{X}^{k+1}-\mathbf{X}^{k}\|_F^2+\|\mathbf{Y}^{k+1}-\mathbf{Y}^{k}\|_F^2\right)\nonumber
\end{align}
\end{lemma}
\begin{proof}
%Noting $\mathbf{Z}^{k+1}\in C_z$ is the minimizer of the constrained $Z$-subproblem in~(\ref{modified_ADMM}), we have  
%%\begin{small}
%\begin{align}\label{sum:1}
%&\mathcal{L}_{\eta}(\mathbf{ P}^{k+1},\mathbf{ Q}^{k+1},\mathbf{Z}^{k+1},\mathbf{X}^{k},\mathbf{Y}^{k},\mathbf{\Lambda}^k,\mathbf{V}^k)\nonumber
%\\&+\frac{1}{2}\|\mathbf{Z}^{k+1}-\mathbf{Z}^k\|_F^2\nonumber
%\\\leq&\mathcal{L}_{\eta}(\mathbf{P}^{k+1},\mathbf{ Q}^{k+1},\mathbf{Z}^k,\mathbf{X}^k,\mathbf{Y}^k,\mathbf{\Lambda}^k,\mathbf{V}^k),
%\end{align}
%%\end{small}
%where the right side of~(\ref{sum:1}) is obtained by letting $\mathbf{Z}$ be $\mathbf{Z}^k\in C_z$ in the $\mathbf{Z}$-subproblem in~(\ref{modified_ADMM}).

Since $\mathbf{X}^{k+1}$ is the minimizer of the $\mathbf{X}$-subproblems in~(\ref{modified_ADMM}) respectively, we obtain
%\begin{small}
\begin{align}
\mathcal{L}_{\eta}(\mathbf{ P}^{k+1},&\mathbf{ Q}^{k+1}\mathbf{X}^{k+1},\mathbf{Y}^{k},\mathbf{\Lambda}^k,\mathbf{V}^k)+\frac{1}{2}\|\mathbf{X}^{k+1}-\mathbf{X}^k\|_F^2 \nonumber
\\ \leq& \mathcal{L}_{\eta}(\mathbf{P}^{k+1},\mathbf{ Q}^{k+1},\mathbf{X}^k,\mathbf{Y}^k,\mathbf{\Lambda}^k,\mathbf{V}^k),\label{sum:2}
\end{align}
Similar result holds for $\mathbf{Y}$.
\end{proof}

\smallskip
\noindent\textbf{Proof of Theorem~\ref{th:exact}:}
Based on the above established lemmas, we now prove Theorem~\ref{th:exact}.
We first introduce some notations.
\begin{align}\label{notation}
&\sigma_0=2\eta^{-1}\sigma_x^{-1}, \sigma_1=2(M_x+1)L_y\eta^{-1}\sigma_x^{-1}+2\eta^{-1}\sigma_y^{-1},\nonumber
\\ &\sigma_2=\frac{1}{2}- 2(L_y+\left(\alpha+1\right)^2)\eta^{-1}\sigma_x^{-1}-2(M_y+1)L_x{\eta^{-1}\sigma_y^{-1}},\nonumber
\\& \sigma_3=\frac{1}{2}- (2L_x+2\left(\alpha+1\right)^2 )\eta^{-1}\sigma_y^{-1},  \nonumber
\\&\sigma_{\ast}=\min\left\{(\eta-\eta_0)/2,\sigma_2-\sigma_0,\sigma_3-\sigma_1\right\},
\end{align}
where constants $\sigma_x, \sigma_y, L_x, L_y$ are given in Lemma~\ref{le:lambda}  $ M_x, M_y$ are given in Assumption~\ref{bound} and $\eta_0$ is given in Lemma~\ref{le:pq}. 
Based on the above notations, we also define
%Let 
\begin{align}\label{func}
\eta_1=& \max\bigg\{  \frac{4+4L_y+4(\alpha+1)^2}{\sigma_x}+\frac{4(M_y+1)L_x}{\sigma_y}, \nonumber
\\&\frac{4L_x+4(\alpha+1)^2}{\sigma_y} +\frac{4(M_x+1)L_y}{\sigma_x}, \eta_0+\frac{1}{4}
\bigg\}.
\end{align}
%\end{small}
%Using Lemma~\ref{le:lambda},~\ref{le:pq} and~\ref{le:xyz}, we turn to prove Theorem~\ref{th:exact}. 
%\begin{proof}[\bf Proof of Theorem~\ref{th:exact}]

Based on~\eqref{lagrangian} and~\eqref{modified_ADMM}, we have
%\begin{small}
\begin{align}\label{eq:laV}
\mathcal{L}_{\eta}(\mathbf{ P}^{k+1},&\mathbf{ Q}^{k+1},\mathbf{X}^{k+1},\mathbf{Y}^{k+1},\mathbf{\Lambda}^{k+1},\mathbf{V}^{k+1})\nonumber
\\&-\mathcal{L}_{\eta}(\mathbf{ P}^{k+1},\mathbf{ Q}^{k+1},\mathbf{X}^{k+1},\mathbf{Y}^{k+1},\mathbf{\Lambda}^k,\mathbf{V}^k)\nonumber
\\=&\|\mathbf{\Lambda}^{k+1}-\mathbf{\Lambda}^{k}\|_F^2/\eta+\|\mathbf{V}^{k+1}-\mathbf{V}^k\|_F^2/\eta.
\end{align}
Combining~(\ref{le:lam}),~(\ref{le:V}),~(\ref{eq:laV}), Lemma~\ref{le:pq} and Lemma~\ref{le:xyz}, we have 
\begin{align}\label{eq:total}
\mathcal{L}_{\eta}&\left(\mathbf{U}^{k+1}\right)-\mathcal{L}_{\eta}\left(\mathbf{U}^{k}\right) \nonumber
\\\leq&-\frac{\eta-\eta_0}{2}\left(\|\mathbf{ P}^{k+1}-\mathbf{ P}^k\|_F^2+\|\mathbf{ Q}^{k+1}-\mathbf{ Q}^k\|_F^2\right)\nonumber
\\&+\sigma_0\|\mathbf{X}^k-\mathbf{X}^{k-1}\|_F^2+\sigma_1\|\mathbf{Y}^k-\mathbf{Y}^{k-1}\|_F^2\nonumber
\\&-\sigma_2\|\mathbf{X}^{k+1}-\mathbf{X}^k\|_F^2-\sigma_3\|\mathbf{Y}^{k+1}-\mathbf{Y}^k\|_F^2.
\end{align} 
Let $\mathbf{{\widetilde U}}^k=\left(\mathbf{ P}^{k},\mathbf{ Q}^{k},\mathbf{X}^{k},\mathbf{Y}^{k},\mathbf{\Lambda}^k,\mathbf{V}^k,\mathbf{X}^{k-1},\mathbf{Y}^{k-1}\right)$ 
and define $\mathcal{\widehat L}_{\eta}$ as
\begin{small}
\begin{align}
\mathcal{\widehat L}_{\eta}\left(\mathbf{{\widetilde U}}^k\right)=&\mathcal{L}_{\eta}\left(\mathbf{U}^k\right)+\sigma_0\|\mathbf{X}^k-\mathbf{X}^{k-1}\|_F^2+\sigma_1\|\mathbf{Y}^k-\mathbf{Y}^{k-1}\|_F^2,   \nonumber
\end{align}
\end{small}
which, in conjunction with~(\ref{eq:total}), indicates that 
\begin{align}\label{inequ:hatl}
\mathcal{\widehat L}_{\eta}&\left(\mathbf{\widetilde U}^{k+1}\right)-\mathcal{\widehat L}_{\eta}\left(\mathbf{\widetilde U}^{k}\right) \nonumber
\\\leq&-\sigma_*\Big(\|\mathbf{ P}^{k+1}-\mathbf{ P}^k\|_F^2+\|\mathbf{ Q}^{k+1}-\mathbf{ Q}^k\|_F^2\nonumber
\\&\quad+\|\mathbf{X}^{k+1}-\mathbf{X}^k\|_F^2+\|\mathbf{Y}^{k+1}-\mathbf{Y}^k\|_F^2\Big).
\end{align}
Combining~(\ref{notation}) and~(\ref{func}), we have $\sigma_*>0$ for any $\eta>\eta_1$, which implies that $\mathcal{\widehat L}_{\eta} (\mathbf{\widetilde U}^{k} )$ is non-increasing. 

Next, we show that the sequences $\left(\mathbf{U}^k\right)$ and $ (\mathbf{\widetilde U}^k )$ are bounded. The boundedness of  the sequence $\left(\mathbf{X}^k,\mathbf{Y}^k\right)$ follows from Assumption~\ref{bound}. Recalling~(\ref{eq:Lam}) and using Lemma~\ref{PD} yield
%\begin{small}
\begin{align}\label{vbound}
\|\mathbf{\Lambda}^k\|_F^2\leq {2(M_x+1)L_y+2M_x+2(\alpha+1)^2M_x}/{\sigma_x}
\end{align}
%\end{small}
A symmetric result holds for $\|\mathbf{V}^k\|_F^2$. Thus, the sequence $\left(\mathbf{\Lambda}^k,\mathbf{V}^k\right)$ is bounded. %For the sequence $\left(\bar P^k ,\bar Q^k\right)$, 
Since  $p(\cdot,\gamma)\geq 0$, we further have
\begin{small}
\begin{align}\label{b2b2}
\mathcal{\widehat L}_{\eta}\left(\mathbf{\widetilde U}^{k}\right)
%\\\geq&\frac{1}{2}\|\Psi_{\Omega}(\mathbf{X}^k(\mathbf{Y}^k)^T-\mathbf{M})\|^{2}_{F}-\frac{1}{2\eta}\left(\|\mathbf{\Lambda}^k\|_F^2+\|\mathbf{V}^k\|_F^2\right)\nonumber
% \\+&\frac{\eta}{2}\|\mathbf{ P}^k-(\mathbf{X}^k)^T\mathbf{E}_x+\frac{1}{\eta}\mathbf{\Lambda}^k\|_F^2+\frac{\alpha}{2}\left(\|\mathbf{X}^k\|_{F}^{2}+\|\mathbf{Y}^k\|_{F}^{2}\right)\nonumber
% \\+&\frac{\eta}{2}\|\mathbf{ Q}^k-(\mathbf{Y}^k)^T\mathbf{E}_y+\frac{1}{\eta}\mathbf{V}^k\|_F^2+\sigma_0\|\mathbf{X}^k-\mathbf{X}^{k-1}\|_F^2+\sigma_1\|\mathbf{Y}^k-\mathbf{Y}^{k-1}\|_F^2\nonumber
\geq&\frac{\eta}{2}\|\mathbf{ P}^k-(\mathbf{X}^k)^T\mathbf{E}_x+\frac{1}{\eta}\mathbf{\Lambda}^k\|_F^2-\frac{1}{2\eta}\left(\|\mathbf{\Lambda}^k\|_F^2+\|\mathbf{V}^k\|_F^2\right)\nonumber
\\+&\frac{\eta}{2}\|\mathbf{ Q}^k-(\mathbf{Y}^k)^T\mathbf{E}_y+\frac{1}{\eta}\mathbf{V}^k\|_F^2.
\end{align}
\end{small}
\hspace{-0.17cm}Since $\mathcal{\widehat L}_{\eta}\left(\mathbf{\widetilde U}^{k}\right)$ is non-increasing,  $\mathcal{\widehat L}_{\eta}\left(\mathbf{\widetilde U}^{1}\right)\geq\mathcal{\widehat L}_{\eta}\left(\mathbf{\widetilde U}^{k}\right)$ for $\forall\, k\in \mathbb{N}$, which, combined with~(\ref{b2b2}), implies
%\begin{small}
\begin{align}
\frac{\eta}{2}\|\mathbf{ P}^k-&(\mathbf{X}^k)^T\mathbf{E}_x+\frac{1}{\eta}\mathbf{\Lambda}^k\|_F^2+\frac{\eta}{2}\|\mathbf{ Q}^k-(\mathbf{Y}^k)^T\mathbf{E}_y+\frac{1}{\eta}\mathbf{V}^k\|_F^2\nonumber
\\\leq&\mathcal{\widehat L}_{\eta}\left(\mathbf{\widetilde U}^{1}\right)+\frac{1}{2\eta}\|\mathbf{\Lambda}^k\|_F^2+\frac{1}{2\eta}\|\mathbf{V}^k\|_F^2<\infty.\nonumber
\end{align}
%\end{small}
Then, using the boundedness of $(\mathbf{X}^k,\mathbf{Y}^k,\mathbf{\Lambda}^k,\mathbf{V}^k)$ and the properties of the Frobenius norm that $\|\mathbf{A-B}\|_F\geq\|\mathbf{A}\|_F-\|\mathbf{B}\|_F$ and $\|\mathbf{AB}\|_F\leq\|\mathbf{A}\|_F\|\mathbf{B}\|_F$, we further obtain
\begin{align}
&\|\mathbf{ P}^k\|_F\leq\|\mathbf{X}^k\|_F\|\mathbf{E}_x\|_F+\frac{1}{\eta}\|\mathbf{\Lambda}^k\|_F<\infty,\nonumber
\\&\|\mathbf{ Q}^k\|_F\leq\|\mathbf{Y}^k\|_F\|\mathbf{E}_y\|_F+\frac{1}{\eta}\|\mathbf{V}^k\|_F<\infty, \nonumber
\end{align}
%\begin{small}
%\begin{align}\label{inequ:bound2}
%&\|\bar P^k\|_F-\|X^k\|_F\|E_x\|_F-\frac{1}{\eta}\|\Lambda^k\|_F\leq\|\bar P^k-X^kE_x\nonumber
%\\&+\frac{1}{\eta}\Lambda^k\|_F<\infty,\nonumber
%&\|\bar Q^k\|_F-\|Y^k\|_F\|E_y\|_F-\frac{1}{\tau}\|V^k\|_F\leq\|\bar Q^k-Y^kE_y\nonumber
%\\&+\frac{1}{\tau}V^k\|_F<\infty,
%\end{align}
%\end{small}
which implies that the sequence $\left(\mathbf{ P}^k,\mathbf{ Q}^k\right)$ is bounded. Combining the results above yields the boundedness of the sequences $\left(\mathbf{U}^k\right)$ and $\big(\mathbf{\widetilde U}^k\big)$. 
Thus, there exists a convergent subsequence $\big(\mathbf{\widetilde U}^{k_i}\big)$. Suppose it converges to $\mathbf{\widetilde U}^{*}$. By the continuity of the function $\mathcal{\widehat L}_{\eta}$, we obtain $\lim\inf_{i\rightarrow\infty}\mathcal{\widehat L}_{\eta}(\mathbf{\widetilde U}^{k_i})\geq\mathcal{\widehat L}_{\eta}(\mathbf{\widetilde U}^{*})$, which, together with \eqref{inequ:hatl} that  $\mathcal{\widehat L}_{\eta}(\mathbf{\widetilde U}^{k})$ is non-increasing, implies that the sequences $\big(\mathcal{\widehat L}_{\eta}(\mathbf{\widetilde U}^{k_i})\big)$ and $\big(\mathcal{\widehat L}_{\eta}(\mathbf{\widetilde U}^{k})\big)$ are bounded below by $\mathcal{\widehat L}_{\eta}(\mathbf{\widetilde U}^{*})$. 
Using the inequality~(\ref{inequ:hatl}), we have, for any $n$
%\begin{small}
\begin{align}\label{exact:L}
\sigma_*&\sum_{k=1}^{n}\Big(\|\mathbf{ P}^{k+1}-\mathbf{ P}^k\|_F^2+\|\mathbf{ Q}^{k+1}-\mathbf{ Q}^k\|_F^2+\|\mathbf{X}^{k+1}-\mathbf{X}^k\|_F^2\nonumber
\\&+\|\mathbf{Y}^{k+1}-\mathbf{Y}^k\|_F^2
\Big)\leq\sum_{k=1}^{n}\left(\mathcal{\widehat L}_{\eta}\left(\mathbf{\widetilde U}^{k}\right)-\mathcal{\widehat L}_{\eta}\left(\mathbf{\widetilde U}^{k+1}\right)\right) \nonumber
\\&\leq\mathcal{\widehat L}_{\eta}\left(\mathbf{\widetilde U}^{1}\right)-\mathcal{\widehat L}_{\eta}\left(\mathbf{\widetilde U}^{*}\right)<\infty, 
\end{align}
%\end{small} 
which, combined with Lemma~\ref{le:lambda},  yields 
%$\sum_{k=1}^{\infty}\|\Lambda^{k+1}-\Lambda^k\|<\infty$, $\sum_{k=1}^{\infty}\|V^{k+1}-V^k\|<\infty$ and hence 
$\sum_{k=1}^{\infty}\|\mathbf{U}^{k+1}-\mathbf{U}^k\|_F^2<\infty$ and  $\lim_{k\rightarrow\infty}\|\mathbf{U}^{k+1}-\mathbf{U}^k\|_F^2\rightarrow0$.

Let $\mathbf{U^*}=(\mathbf{ P^*, Q^*,X^*,Y^*,\Lambda^*,V^*})$ denote any cluster point of sequence $\left(\mathbf{U}^k\right)$. 
% By the constraint $\Psi_{\Omega}(Z)=\Psi_{\Omega}(M)$, only the free entries $\notin \Omega$ of $Z$, denoted by $\widetilde Z$, are updated in each iteration. applying Fermat's rule
Applying Fermat's rule to (\ref{modified_ADMM}) yields
\small
\begin{align}
\mathbf{0}=&\nabla_{\mathbf{X}}\Big(\frac{1}{2}\|\Psi_\Omega(\mathbf{X}^{k+1}(\mathbf{Y}^k)^T-\mathbf{M})\|_F^2+\frac{\alpha}{2}\|\mathbf{X}^{k+1}\|_F^2\Big)\nonumber
\\+&\mathbf{E}_x(\mathbf{\Lambda}^k)^T+\Big(\mathbf{X}^{k+1}-\mathbf{X}^k\Big)-\eta\mathbf{E}_x\left((\mathbf{ P}^{k+1})^T-\mathbf{E}_x^T\mathbf{X}^{k+1}\right),\nonumber
\\\mathbf{0}=&\nabla_{\mathbf{Y}}\Big(\frac{1}{2}\|\Psi_\Omega(\mathbf{X}^{k+1}(\mathbf{Y}^{k+1})^T-\mathbf{M})\|_F^2+\frac{\alpha}{2}\|\mathbf{Y}^{k+1}\|_F^2\Big)\nonumber
\\+&\mathbf{E}_y(\mathbf{V}^k)^T+\Big(\mathbf{Y}^{k+1}-\mathbf{Y}^k\Big)-\eta\mathbf{E}_y\left((\mathbf{ Q}^{k+1})^T-\mathbf{E}_y^T\mathbf{Y}^{k+1}\right),\nonumber
\end{align}
\normalsize
and
\small
\begin{align}
\mathbf{0}\in&\partial_{\mathbf{ P}}{\sum_{l \in \varepsilon_{X} }w_{l}\hspace{0.04cm}p (\mathbf{p}_{l}^{k+1},\gamma_{X} )}+\mathbf{\Lambda}^k+\eta\left(\mathbf{ P}^{k+1}-(\mathbf{X}^k)^T\mathbf{E}_x\right),\nonumber
\\\mathbf{0}\in&\partial_{\mathbf{ Q}}{\sum_{j \in \varepsilon_{Y}}u_{j}\hspace{0.04cm}p (\mathbf{q}_{j}^{k+1},\gamma_{Y} )}+\mathbf{V}^k+\eta\left(\mathbf{ Q}^{k+1}-(\mathbf{Y}^k)^T\mathbf{E}_y\right), \nonumber
\end{align}
\normalsize
which, in conjunction with~\eqref{modified_ADMM} 
%\begin{align}
%&\mathbf{\Lambda}^{k+1}-\mathbf{\Lambda}^k=\eta\Big(\mathbf{ P}^{k+1}-(\mathbf{X}^{k+1})^T\mathbf{E}_x\Big)\;\text{and}\;\mathbf{V}^{k+1}-\mathbf{V}^k=\eta\left(\mathbf{ Q}^{k+1}-(\mathbf{Y}^{k+1})^T\mathbf{E}_y\right)\nonumber
%\end{align}
%and using~(\ref{stationary}) 
and $\|\mathbf{U}^{k+1}-\mathbf{U}^k\|_F^2\rightarrow0$, implies  
%\begin{small}
%\begin{align}\label{eq:optimal}
$\mathbf{ P^{*}}-(\mathbf{X}^{*})^T\mathbf{E}_x=\mathbf{0}$, $\mathbf{ Q}^{*}-(\mathbf{Y}^{*})^T\mathbf{E}_y=\mathbf{0}$ and 
%\end{align}
%\end{small}
\small
\begin{align}%\label{con:stationary}
&\mathbf{0}\in\partial_{\mathbf{ P}}{\sum_{l\in\varepsilon_{X}}w_{l}\hspace{0.04cm}p\left(\mathbf{p}_{l}^{*},\gamma_{X}\right)}+\mathbf{\Lambda}^*,\;\mathbf{0}\in\partial_{\mathbf{ Q}}{\sum_{j\in\varepsilon_{Y}}u_{j}\hspace{0.04cm}p\left(\mathbf{q}_{j}^{*},\gamma_{Y}\right)}+\mathbf{V}^*,\nonumber
%\end{align}
%\begin{align}
\\&\mathbf{0}=\nabla_{\mathbf X}\left(\frac{1}{2}\|\Psi_\Omega(\mathbf{X}^*(\mathbf{Y}^*)^T-\mathbf{M})\|_F^2+\frac{\alpha}{2}\|\mathbf{X}^*\|_F^2\right)+\mathbf{E}_x(\mathbf{\Lambda}^*)^T,\nonumber
\\&\mathbf{0}=\nabla_{\mathbf Y}\left(\frac{1}{2}\|\Psi_\Omega(\mathbf{X}^{*}(\mathbf{Y}^*)^T-\mathbf{M})\|_F^2+\frac{\alpha}{2}\|\mathbf{Y}^*\|_F^2\right)+\mathbf{E}_y(\mathbf{V}^*)^T,\nonumber
\end{align}
\normalsize
which implies that $\mathbf{0}\in\partial\mathcal{L}_{\eta}\left(\mathbf{U}^*\right)$ and thus
$\mathbf{U}^*$ is a stationary point of  $\mathcal{L}_{\eta}$. The above result establishes a subsequential convergence for the  proposed algorithm.

We next show that when the penalty function is semi-algebraic, we enable to provide a global convergence for the proposed algorithm by using the K\L~property. Let $\mathcal{Q}$ be a set of all cluster points of the sequence $(\mathbf{\widetilde U}^k)$. For any point $\mathbf{\widetilde U}^*$ in $\mathcal{Q}$, let $(\mathbf{\widetilde U}^{k_j})$ be a subsequence of $(\mathbf{\widetilde U}^k)$ converging to the point $\mathbf{\widetilde U}^*$. 
As we prove before (see~\eqref{inequ:hatl} and the text above \eqref{exact:L}),  the sequence $(\mathcal{\widehat L}_{\eta}(\mathbf{\widetilde U}^{k+1}) )$  is non-increasing and lower-bounded, and hence is convergent (without loss of generality, we suppose it converges to $C$). Then, we have 
%\begin{align}
%\sigma_*\|\mathbf{\widetilde U}^{k+1}-\mathbf{\widetilde U}^k\|^2 \leq \mathcal{\widehat L}_{\eta}(\mathbf{\widetilde U}^{k})-\mathcal{\widehat L}_{\eta}(\mathbf{\widetilde U}^{k+1})
%\end{align}
\begin{align}\label{wocaosss}
\mathcal{\widehat L}_{\eta}(\mathbf{\widetilde U}^{*}) = \lim_{j\rightarrow \infty} \mathcal{\widehat L}_{\eta}(\mathbf{\widetilde U}^{k_j})=\lim_{k\rightarrow\infty}\mathcal{\widehat L}_{\eta}(\mathbf{\widetilde U}^{k})=C,
\end{align}
which implies that the function $\mathcal{\widehat L}_{\eta}(\cdot)$ is constant on $\mathcal{Q}$.  Suppose that the penalty function $p(\cdot,\gamma)$ is semi-algebraic. Based on Example 4 in \cite{bolte2014proximal}, we have $\|\cdot\|_2$ is also semi-algebraic. Then, according to the item 4  in Example 2 in \cite{bolte2014proximal} that composition of semi-algebraic functions are semi-algebraic,  we have the function $p(\|\cdot\|_2,\gamma)$ is also semi-algebraic. Now, note that  all individual terms in the  augmented Lagrangian function $\mathcal{\widehat L}_{\eta}(\cdot)$ are semi-algebraic. Based on item 3 in Example 2 in \cite{bolte2014proximal} that finite sums of semi-algebraic functions is semi-algebraic, we have that $\mathcal{\widehat L}_{\eta}(\cdot)$ is semi-algebraic, and hence  satisfies the K\L\, property (see \cite[Theorem 3]{bolte2014proximal}) at all points in $\mathcal{Q}$. Based on K\L\,inequality in \cite[Lemma 2.2]{wang2018convergence}, we have, there exist $r>0, \delta>0$ and a smooth concave function $\phi: [0,r]\rightarrow \mathbb{R}^+$ (which satisfies three conditions in~\cite[(a),(b),(c) in Definition 2.3]{wang2018convergence}) such that 
\begin{align}\label{diaosisss}
\text{dist}(0,\partial\mathcal{\widehat L}_{\eta}(\mathbf{\widetilde U}))\phi ' (\mathcal{\widehat L}_{\eta}(\mathbf{\widetilde U})-\mathcal{\widehat L}_{\eta}(\mathbf{\widetilde U}^{*})) \geq 1,
\end{align}
for all $\mathbf{\widetilde U}$ satisfying $\text{dist}(\mathbf{\widetilde U},\mathcal{Q})<\delta$ and $\mathcal{\widehat L}_{\eta}(\mathbf{\widetilde U}^{*})<\mathcal{\widehat L}_{\eta}(\mathbf{\widetilde U})<\mathcal{\widehat L}_{\eta}(\mathbf{\widetilde U}^{*})+r$, where $\text{dist}(\mathbf{x},S):=\inf \{\|\mathbf{x-y}\|:\mathbf{y}\in S\}$. 
 Based on the definition of $\mathcal{Q}$, we have $\lim_{k\rightarrow \infty}\text{dist}(\mathbf{\widetilde U}^k,\mathcal{Q})=0$, which, combined with the \eqref{wocaosss} that $\lim_{k\rightarrow\infty}\mathcal{\widehat L}_{\eta}(\mathbf{\widetilde U}^{k})\rightarrow \mathcal{\widehat L}_{\eta}(\mathbf{\widetilde U}^{*}) $, implies that, there exists an integer $k_0$ such that for any $k\geq k_0$, $\text{dist}(\mathbf{\widetilde U}^k,\mathcal{Q})<\delta$ and $\mathcal{\widehat L}_{\eta}(\mathbf{\widetilde U}^{*})<\mathcal{\widehat L}_{\eta}(\mathbf{\widetilde U}^k)<\mathcal{\widehat L}_{\eta}(\mathbf{\widetilde U}^{*})+r$. Then, we have, for any $k>k_0$, 
%at each point in $\mathcal{Q}$ (see   \cite[Definition 2.3]{wang2018convergence}, and thus $\mathcal{\widehat L}_\eta$ satisfy the K\L \,property.
% In the meanwhile, 
%based on  the results in \cite{li2018calculus} on MCP and SCAD regularization, it can be verified that $\mathcal{\widehat L}_\eta$ also satisfies the K\L \,property when $p(\cdot,\gamma)$ is either MCP or SCAD. 
% Combining the above two facts with \cite[Lemma 2]{wang2018convergence} and \cite[Theorem  3.6]{wang2018convergence}, we have,  there exists an integer $k_0$, a positive constant $r$ and a smooth concave function $\phi: [0,r]\rightarrow \mathbb{R}^+$ such that,
%  for any $k\geq k_0$ and any $\mathbf{\widetilde U}^{*}$ in $\mathcal{Q}$, $\text{dist}(0,\partial\mathcal{\widehat L}_{\eta}(\mathbf{\widetilde U}^{k}))<$
\begin{align}\label{diaosi}
\text{dist}(0,\partial\mathcal{\widehat L}_{\eta}(\mathbf{\widetilde U}^{k}))\phi ' (\mathcal{\widehat L}_{\eta}(\mathbf{\widetilde U}^{k})-\mathcal{\widehat L}_{\eta}(\mathbf{\widetilde U}^{*})) \geq 1,
\end{align}
and $\mathcal{\widehat L}_{\eta}(\mathbf{\widetilde U}^{k}) > \mathcal{\widehat L}_{\eta}(\mathbf{\widetilde U}^{*})$, where dist$(\mathbf{x}_0,\partial f(\mathbf{x})):=\inf\{\|\mathbf{x}_0-\mathbf{y}\|: \mathbf{y}\in \partial f(\mathbf{x})\}$.

Using an approach similar to~\cite[Lemma 3.2]{wang2018convergence}, we have, there exists a constant $\tau_0>0$ such that 
\begin{align}\label{muops}
\text{dist}&(0,\partial\mathcal{\widehat L}_{\eta}(\mathbf{\widetilde U}^{k}))\leq \tau_0 \big( \|\mathbf{X}^{k}-\mathbf{X}^{k-1}\|_F+\|\mathbf{Y}^{k}-\mathbf{Y}^{k-1}\|_F \nonumber
\\&+\|\mathbf{ P}^{k}-\mathbf{ P}^{k-1}\|_F+\|\mathbf{ Q}^{k}-\mathbf{ Q}^{k-1}\|_F\nonumber
\\&+\|\mathbf{X}^{k-1}-\mathbf{X}^{k-2}\|_F+\|\mathbf{Y}^{k-1}-\mathbf{Y}^{k-2}\|_F\big),
\end{align}
which, in conjunction with~\eqref{diaosi}, implies that 
\begin{small}
\begin{align}\label{diaodiao}
\frac{1}{\phi ' (\mathcal{\widehat L}_{\eta}(\mathbf{\widetilde U}^{k})-\mathcal{\widehat L}_{\eta}(\mathbf{\widetilde U}^{*}))}&\leq \tau_0 \big( \|\mathbf{X}^{k}-\mathbf{X}^{k-1}\|_F+\|\mathbf{Y}^{k}-\mathbf{Y}^{k-1}\|_F \nonumber
\\+\|\mathbf{ P}^{k}-\mathbf{ P}^{k-1}\|_F+&\|\mathbf{ Q}^{k}-\mathbf{ Q}^{k-1}\|_F\nonumber
\\+\|\mathbf{X}^{k-1}-\mathbf{X}^{k-2}\|_F&+\|\mathbf{Y}^{k-1}-\mathbf{Y}^{k-2}\|_F\big).
\end{align} 
\end{small}
Based on~\eqref{inequ:hatl},  we have
\begin{align}\label{mooom}
&\sigma_*\Big(\|\mathbf{ P}^{k+1}-\mathbf{ P}^k\|_F^2+\|\mathbf{ Q}^{k+1}-\mathbf{ Q}^k\|_F^2+\|\mathbf{X}^{k+1}-\mathbf{X}^k\|_F^2\nonumber
\\&+\|\mathbf{Y}^{k+1}-\mathbf{Y}^k\|_F^2\Big)\leq \mathcal{\widehat L}_{\eta} \left(\mathbf{\widetilde U}^{k}\right)-\mathcal{\widehat L}_{\eta}\left(\mathbf{\widetilde U}^{k+1}\right) \nonumber
\\&\overset{\text{(i)}}\leq \frac{\phi (\mathcal{\widehat L}_{\eta}(\mathbf{\widetilde U}^{k})-\mathcal{\widehat L}_{\eta}(\mathbf{\widetilde U}^{*}))-\phi (\mathcal{\widehat L}_{\eta}(\mathbf{\widetilde U}^{k+1})-\mathcal{\widehat L}_{\eta}(\mathbf{\widetilde U}^{*}))}{\phi ' (\mathcal{\widehat L}_{\eta}(\mathbf{\widetilde U}^{k})-\mathcal{\widehat L}_{\eta}(\mathbf{\widetilde U}^{*}))}
\end{align} 
where (i) follows from the concavity of the function $\phi(\cdot)$. 
To simplify notations, we define 
\begin{align*}
D(\mathbf{\widetilde U}^{k},\mathbf{\widetilde U}^{k+1})=&\phi (\mathcal{\widehat L}_{\eta}(\mathbf{\widetilde U}^{k})-\mathcal{\widehat L}_{\eta}(\mathbf{\widetilde U}^{*}))
\\&-\phi (\mathcal{\widehat L}_{\eta}(\mathbf{\widetilde U}^{k+1}) -\mathcal{\widehat L}_{\eta}(\mathbf{\widetilde U}^{*})).
\end{align*}
Then, combining~\eqref{diaodiao} and~\eqref{mooom} yields
\begin{align*}
\sigma_*&\Big(\|\mathbf{ P}^{k+1}-\mathbf{ P}^k\|_F^2+\|\mathbf{ Q}^{k+1}-\mathbf{ Q}^k\|_F^2+\|\mathbf{X}^{k+1}-\mathbf{X}^k\|_F^2\nonumber
\\&+\|\mathbf{Y}^{k+1}-\mathbf{Y}^k\|_F^2\Big)\leq \mathcal{\widehat L}_{\eta} \left(\mathbf{\widetilde U}^{k}\right)-\mathcal{\widehat L}_{\eta}\left(\mathbf{\widetilde U}^{k+1}\right) \nonumber
\\&\leq \tau_0 D(\mathbf{\widetilde U}^{k},\mathbf{\widetilde U}^{k+1}) \big( \|\mathbf{X}^{k}-\mathbf{X}^{k-1}\|_F+\|\mathbf{Y}^{k}-\mathbf{Y}^{k-1}\|_F \nonumber
\\&+\|\mathbf{ P}^{k}-\mathbf{ P}^{k-1}\|_F+\|\mathbf{ Q}^{k}-\mathbf{ Q}^{k-1}\|_F\nonumber
\\&+\|\mathbf{X}^{k-1}-\mathbf{X}^{k-2}\|_F+\|\mathbf{Y}^{k-1}-\mathbf{Y}^{k-2}\|_F\big), 
\end{align*} 
which, using $\sqrt{n}(\sum_{i=1}^n\mathbf{A}_i\|_F^2)^{1/2}\geq \sum_{i=1}^n\|\mathbf{A}_i\|_F$, yields
\begin{align*}
&3(\|\mathbf{ P}^{k+1}-\mathbf{ P}^k\|_F+\|\mathbf{ Q}^{k+1}-\mathbf{ Q}^k\|_F+\|\mathbf{X}^{k+1}-\mathbf{X}^k\|_F\nonumber
\\&+\|\mathbf{Y}^{k+1}-\mathbf{Y}^k\|_F)\leq \mathcal{\widehat L}_{\eta} \left(\mathbf{\widetilde U}^{k}\right)-\mathcal{\widehat L}_{\eta}\left(\mathbf{\widetilde U}^{k+1}\right) \nonumber
\\&\leq 6\sqrt{\frac{\tau_0 D(\mathbf{\widetilde U}^{k},\mathbf{\widetilde U}^{k+1})}{\sigma_*}} \big( \|\mathbf{X}^{k}-\mathbf{X}^{k-1}\|_F+\|\mathbf{Y}^{k}-\mathbf{Y}^{k-1}\|_F \nonumber
\\&+\|\mathbf{ P}^{k}-\mathbf{ P}^{k-1}\|_F+\|\mathbf{ Q}^{k}-\mathbf{ Q}^{k-1}\|_F\nonumber
\\&+\|\mathbf{X}^{k-1}-\mathbf{X}^{k-2}\|_F+\|\mathbf{Y}^{k-1}-\mathbf{Y}^{k-2}\|_F\big)^{1/2}, 
\end{align*}
which, in conjunction with the inequality $2ab\leq a^2+b^2$, yields
\begin{align*}
&3(\|\mathbf{ P}^{k+1}-\mathbf{ P}^k\|_F+\|\mathbf{ Q}^{k+1}-\mathbf{ Q}^k\|_F+\|\mathbf{X}^{k+1}-\mathbf{X}^k\|_F\nonumber
\\&+\|\mathbf{Y}^{k+1}-\mathbf{Y}^k\|_F)\leq \frac{9\tau_0 D(\mathbf{\widetilde U}^{k},\mathbf{\widetilde U}^{k+1})}{\sigma_*}\nonumber
\\&+  \|\mathbf{X}^{k}-\mathbf{X}^{k-1}\|_F+\|\mathbf{Y}^{k}-\mathbf{Y}^{k-1}\|_F +\|\mathbf{ P}^{k}-\mathbf{ P}^{k-1}\|_F\nonumber
\\&+\|\mathbf{ Q}^{k}-\mathbf{ Q}^{k-1}\|_F+\|\mathbf{X}^{k-1}-\mathbf{X}^{k-2}\|_F+\|\mathbf{Y}^{k-1}-\mathbf{Y}^{k-2}\|_F. 
\end{align*}
Telescoping the above inequality over $k$ from $k_0$ to $K$ yields
\begin{small}
\begin{align*}
&3\sum_{k=k_0}^K(\|\mathbf{ P}^{k+1}-\mathbf{ P}^k\|_F+\|\mathbf{ Q}^{k+1}-\mathbf{ Q}^k\|_F+\|\mathbf{X}^{k+1}-\mathbf{X}^k\|_F\nonumber
\\&+\|\mathbf{Y}^{k+1}-\mathbf{Y}^k\|_F)\leq \sum_{k=k_0}^K \frac{9\tau_0 D(\mathbf{\widetilde U}^{k},\mathbf{\widetilde U}^{k+1})}{\sigma_*}\nonumber
\\&+ \sum_{k=k_0}^K( \|\mathbf{X}^{k}-\mathbf{X}^{k-1}\|_F+\|\mathbf{Y}^{k}-\mathbf{Y}^{k-1}\|_F +\|\mathbf{ P}^{k}-\mathbf{ P}^{k-1}\|_F\nonumber
\\&+\|\mathbf{ Q}^{k}-\mathbf{ Q}^{k-1}\|_F+\|\mathbf{X}^{k-1}-\mathbf{X}^{k-2}\|_F+\|\mathbf{Y}^{k-1}-\mathbf{Y}^{k-2}\|_F). 
\end{align*} 
\end{small}
Rearranging the terms in the above inequality, we obtain 
\begin{small}
\begin{align}\label{disoos}
&\sum_{k=k_0}^K(2\|\mathbf{ P}^{k+1}-\mathbf{ P}^k\|_F+2\|\mathbf{ Q}^{k+1}-\mathbf{ Q}^k\|_F+\|\mathbf{X}^{k+1}-\mathbf{X}^k\|_F\nonumber
\\&+\|\mathbf{Y}^{k+1}-\mathbf{Y}^k\|_F)\leq \sum_{k=k_0}^K \frac{9\tau_0 D(\mathbf{\widetilde U}^{k},\mathbf{\widetilde U}^{k+1})}{\sigma_*}+ 2 \|\mathbf{X}^{k_0}-\mathbf{X}^{k_0-1}\|_F\nonumber
\\&+\|\mathbf{X}^{k_0-1}-\mathbf{X}^{k_0-2}\|_F +\|\mathbf{ P}^{k_0}-\mathbf{ P}^{k_0-1}\|_F+\|\mathbf{ Q}^{k_0}-\mathbf{ Q}^{k_0-1}\|_F\nonumber
\\&+2\|\mathbf{Y}^{k_0}-\mathbf{Y}^{k_0-1}\|_F+\|\mathbf{Y}^{k_0-1}-\mathbf{Y}^{k_0-2}\|_F, 
\end{align}
\end{small}
\hspace{-0.12cm}Note  that 
\begin{align}\label{dssadasd}
&\sum_{k=k_0}^K \frac{9\tau_0 D(\mathbf{\widetilde U}^{k},\mathbf{\widetilde U}^{k+1})}{\sigma_*} \nonumber
\\&=\frac{9\tau_0}{\sigma_*}(\phi (\mathcal{\widehat L}_{\eta}(\mathbf{\widetilde U}^{k_0})-\mathcal{\widehat L}_{\eta}(\mathbf{\widetilde U}^{*}))-\phi (\mathcal{\widehat L}_{\eta}(\mathbf{\widetilde U}^{K+1})-\mathcal{\widehat L}_{\eta}(\mathbf{\widetilde U}^{*})))  \nonumber
\\&\overset{\text{(i)}}<\frac{9\tau_0}{\sigma_*}\phi (\mathcal{\widehat L}_{\eta}(\mathbf{\widetilde U}^{k_0})-\mathcal{\widehat L}_{\eta}(\mathbf{\widetilde U}^{*}))
\end{align}
where (i) follows from the fact that $r>\mathcal{\widehat L}_{\eta}(\mathbf{\widetilde U}^{K+1})-\mathcal{\widehat L}_{\eta}(\mathbf{\widetilde U}^{*})\geq 0$ and $\phi(\mathbf{x})\geq 0$ for $0<\mathbf{x}<r$ (Definition of function $\phi$ before). Combining~\eqref{disoos},~\eqref{dssadasd} and letting $K\rightarrow\infty$, we obtain
\begin{align*}
&\sum_{k=0}^\infty(\|\mathbf{ P}^{k+1}-\mathbf{ P}^k\|_F+\|\mathbf{ Q}^{k+1}-\mathbf{ Q}^k\|_F+\|\mathbf{X}^{k+1}-\mathbf{X}^k\|_F\nonumber
\\&+\|\mathbf{Y}^{k+1}-\mathbf{Y}^k\|_F)<\infty,
\end{align*} 
which, in conjunction with Lemma~\ref{le:lambda}, also implies that $\sum_{k=0}^\infty (\|\mathbf{\Lambda}^{k+1}-\mathbf{\Lambda}^k\|_F+\|\mathbf{V}^{k+1}-\mathbf{V}^k\|_F)<\infty$. Combining the above two facts, we have $\sum_{k=0}^\infty \|\mathbf{U}^{k+1}-\mathbf{U}^k\|_F<\infty$, which means that the sequence $(\mathbf{U}^k)$ is convergent. 
 
%\end{proof}

\subsection{Proof of Theorem~\ref{th:inexact}}\label{apen:D}
\begin{proof}
Note that $\mathbf{\widehat X}^{k+1}$ can be regarded as the minimizer of the following subproblem
%\begin{small}
\begin{align}\label{hatX}
\mathbf{\widehat X}^{k+1}=&\arg \min\limits_{\mathbf{X}}\left (f(\text{vec}(\mathbf{X}^T))-(\mathbf{t}_x^{k+1})^T\text{vec}(\mathbf{X}^T)\right),
\end{align}
where the function $f(\cdot)$ is given by~(\ref{cg}). 
%\end{small}
Then, using a similar approach to~(\ref{eq:Lam}), we can obtain
\begin{align}\label{hat:x}
\text{vec}(\mathbf{\Lambda}^{k+1}\mathbf{E}_x^T)=&(\mathbf{G}_y^k+(\alpha+1)\mathbf{I}_{nd})\text{vec}((\mathbf{\widehat X}^{k+1})^T)-(\mathbf{b}_y^k)^T\nonumber
\\&-\mathbf{t}_x^{k+1}-\text{vec}((\mathbf{\widehat X}^{k})^T).
\end{align}
Similarly to~(\ref{le:lam}) and based on~(\ref{hat:x}), we can obtain
\begin{align}\label{inexact:Lamb}
& \|\mathbf{\Lambda}^{k+1}-\mathbf{\Lambda}^k\|_F^2 \nonumber
\\\leq&\frac{2}{\sigma_x}\|\mathbf{\widehat X}^k-\mathbf{\widehat X}^{k-1}\|_F^2+\frac{2(M_x+1)L_y}{\sigma_x}\|\mathbf{Y}^k-\mathbf{Y}^{k-1}\|_F^2\nonumber
\\&+\frac{2L_y+2(\alpha+1)^2}{\sigma_x}\|\mathbf{\widehat X}^{k+1}-\mathbf{\widehat X}^k\|_F^2\nonumber
\\&+\frac{2}{\sigma_x}\left(\|\mathbf{t}_x^{k}\|_2^2+\|\mathbf{t}_x^{k+1}\|_2^2\right).
\end{align}
A symmetric result holds for $\|\mathbf{V}^{k+1}-\mathbf{V}^k\|_F^2$. Using a similar approach to~(\ref{sum:2}) and recalling~(\ref{hatX}), we have
%\begin{small}
\begin{align}\label{inexact:L}
\mathcal{L}_{\eta}& (\mathbf{ P}^{k+1},\mathbf{ Q}^{k+1},\mathbf{\widehat X}^{k+1},\mathbf{\widehat Y}^{k},\mathbf{\Lambda}^k,\mathbf{V}^k) \nonumber
\\&-\mathcal{L}_{\eta}\left(\mathbf{ P}^{k+1},\mathbf{ Q}^{k+1},\mathbf{X}^k,\mathbf{Y}^k,\mathbf{\Lambda}^k,\mathbf{V}^k\right)\nonumber
\\\leq&-\frac{1}{2}\|\mathbf{\widehat X}^{k+1}-\mathbf{\widehat X}^k\|_F^2+\mathbf{t}_x^{k+1}(\text{vec}(\mathbf{\widehat X}^{k+1})-\text{vec}(\mathbf{\widehat X}^{k}))\nonumber
\\\overset{(i)}\leq&-\frac{1}{4}\|\mathbf{\widehat X}^{k+1}-\mathbf{\widehat X}^{k}\|_F^2+\frac{1}{2}\|\mathbf{t}_x^{k+1}\|_2^2,
\end{align}
%\end{small} 
where $(i)$ follows from the fact that $\|\mathbf{A+B}\|_F^2\leq 2(\|\mathbf{A}\|_F^2+\|\mathbf{B}\|_F^2)$. A symmetrical result holds for $\mathbf{\widehat Y}^{k+1}$. 
Next, we show that the sequence $\big(\mathbf{\widehat U}^k\big)$ is bounded. The boundedness of $\big(\mathbf{\widehat X}^k,\mathbf{\widehat Y}^k\big)$ follows from Property \ref{bound}.  Similar to (\ref{vbound}), using (\ref{hat:x}) and Property \ref{PD} yields the boundedness of $\left(\mathbf{\Lambda}^k,\mathbf{V}^k\right)$, which, in conjunction with the last two equalities in~(\ref{modified_ADMM}), yields the boundedness of $\left(\mathbf{ P}^k,\mathbf{ Q}^k\right)$.
% the boundedness follows from the coercivity of the objective function~(\ref{model}) by noting that $\lim_{\|\bar P\|_F\rightarrow\infty}\left[\sum_{l \in \varepsilon_{X}}w_{l}p\left(P_{l},\gamma_{X}\right)\right]=\infty$ and $\lim_{\|\bar Q\|_F\rightarrow\infty}\left[\sum_{l \in \varepsilon_{Y}}v_{l}p\left(Q_{l},\gamma_{Y}\right)\right]=\infty$. If the penalty function is MCP, we can introduce extra regularizers $\lambda_P\|\bar P\|_F^2/2$ and $\lambda_Q\|\bar Q\|_F^2/2$ in the objective function~(\ref{model}) to guarantee the boundedness of $\{\bar P^k,\bar Q^k\}$. 
These facts imply  the boundedness of the sequence of $\big(\mathbf{\widehat U}^k\big)$. 

Since the function $\mathcal{\widehat L}_{\eta}$ is continuous and coercive, by the boundedness of $\left(\mathbf{\widehat U}^k\right)$, we have, there exists $M_{l}>0$ such that $\left|\mathcal{\widehat L}_{\eta}(\mathbf{\widehat U}^k)\right|<M_l$ for $\forall$ $k$. 
Combining~(\ref{inexact:L}), Lemma~\ref{le:pq} and Lemma~\ref{le:xyz}, we have
%\begin{small}
\begin{align}\label{diaosssi}
\Big(\sigma_*&-\frac{1}{4}\Big)\sum_{k=1}^{n}\Big(\|\mathbf{ P}^{k+1}-\mathbf{ P}^k\|_F^2+\|\mathbf{ Q}^{k+1}-\mathbf{ Q}^k\|_F^2 \nonumber
\\&+\|\mathbf{\widehat X}^{k+1}-\mathbf{\widehat X}^k\|_F^2+\|\mathbf{\widehat Y}^{k+1}-\mathbf{\widehat Y}^k\|_F^2\Big)\nonumber
\\ \leq&\mathcal{\widehat L}_{\eta}\left(\mathbf{\widetilde U}^{1}\right)-\mathcal{\widehat L}_{\eta}\big(\mathbf{\widetilde U}^{n+1}\big)+\frac{1}{2}\sum_{k=1}^{n+1}\big(\|\mathbf{t}_x^k\|_2^2+\|\mathbf{t}_y^k\|_2^2\big)<\infty
\end{align}
%\end{small}
which, combined with~(\ref{notation}),~(\ref{func}) and $\eta>2\eta_1$, implies $\sigma_*>1/4$ and hence  
%letting $n\rightarrow\infty$ in~(\ref{tyxx}) implies
%%\begin{small}
\begin{align}\label{ineq:hat}
\sum_{k=1}^{\infty}&\big(\|\mathbf{ P}^{k+1}-\mathbf{ P}^{k}\|_F^2+\|\mathbf{ Q}^{k+1}-\mathbf{ Q}^k\|_F^2+\|\mathbf{\widehat X}^{k+1}-\mathbf{\widehat X}^k\|_F^2 \nonumber
\\&+\|\mathbf{\widehat Y}^{k+1}-\mathbf{\widehat Y}^k\|_F^2\big)<\infty,
\end{align}
%\end{small}
Combining~(\ref{inexact:Lamb}),~\eqref{ineq:hat} and  the fact that $\sum_{k=1}^{\infty}\|\mathbf{t}_x^k\|_2^2<\infty$ implies
$\sum_{k=1}^{\infty}\|\mathbf{\Lambda}^{k+1}-\mathbf{\Lambda}^k\|_F^2<\infty$. 
Similarly, we can obtain $\sum_{k=1}^{\infty}\|\mathbf{V}^{k+1}-\mathbf{V}^k\|_F^2<\infty$. Based on the above results, we conclude that  $\lim_{k\rightarrow\infty}\|\mathbf{\widehat U}^{k+1}-\mathbf{\widehat U}^k\|_F^2=0$. 
%Next, we show any cluster point of $\{\widehat U^k\}$, denoted as
Let $\mathbf{\widehat U}^*=(\mathbf{ P^*, Q^*,\widehat X^*,\widehat Y^*,\Lambda^*,V^*})$ be a limit point of $\{\mathbf{\widehat U}^k\}$. 
Using steps similar to those in the proof of Theorem~\ref{th:exact} and using $\|\mathbf{t}_x^{k+1}\|_2\rightarrow0$ we have $\mathbf{0}\in\partial\mathcal{L}_{\eta}(\mathbf{\widehat U}^*)$ and $\mathbf{\widehat U}^*$ is a stationary point of $\mathcal{L}_{\eta}$.

Using an approach similar to~\eqref{muops}, we obtain, there exists a positive constant such that %\eqref{diaosi} 
\begin{align}
&\text{dist}(0,\partial\mathcal{\widehat L}_{\eta}(\mathbf{\widetilde U}^{k}))\leq \gamma_0 \big( \|\mathbf{X}^{k}-\mathbf{X}^{k-1}\|_F+\|\mathbf{Y}^{k}-\mathbf{Y}^{k-1}\|_F \nonumber
\\&+\|\mathbf{ P}^{k}-\mathbf{ P}^{k-1}\|_F+\|\mathbf{ Q}^{k}-\mathbf{ Q}^{k-1}\|_F+\|\mathbf{X}^{k-1}-\mathbf{X}^{k-2}\|_F\nonumber
\\&+\|\mathbf{Y}^{k-1}-\mathbf{Y}^{k-2}\|_F+\|\mathbf{t}_x^{k}\|_2+\|\mathbf{t}_y^{k}\|_2\big),
\end{align}
which, in conjunction with~\eqref{diaosi} and~\eqref{diaosssi} and using an approach similar to~\eqref{dssadasd}, yields 
\begin{small}
\begin{align*}
&\sum_{k=k_0}^K(2\|\mathbf{ P}^{k+1}-\mathbf{ P}^k\|_F+2\|\mathbf{ Q}^{k+1}-\mathbf{ Q}^k\|_F+\|\mathbf{X}^{k+1}-\mathbf{X}^k\|_F\nonumber
\\&+\|\mathbf{Y}^{k+1}-\mathbf{Y}^k\|_F)\leq \frac{9\gamma_0 \phi (\mathcal{\widehat L}_{\eta}(\mathbf{\widetilde U}^{k_0})-\mathcal{\widehat L}_{\eta}(\mathbf{\widetilde U}^{*}))
}{\sigma_*-1/4}\nonumber
\\&+ 2 \|\mathbf{X}^{k_0}-\mathbf{X}^{k_0-1}\|_F+\|\mathbf{X}^{k_0-1}-\mathbf{X}^{k_0-2}\|_F +\|\mathbf{ P}^{k_0}-\mathbf{ P}^{k_0-1}\|_F\nonumber
\\&+\|\mathbf{ Q}^{k_0}-\mathbf{ Q}^{k_0-1}\|_F+2\|\mathbf{Y}^{k_0}-\mathbf{Y}^{k_0-1}\|_F+\|\mathbf{Y}^{k_0-1}-\mathbf{Y}^{k_0-2}\|_F, \nonumber
\\&+\sum_{k=k_0}^K (\|\mathbf{t}_x^{k}\|_2+\|\mathbf{t}_y^{k}\|_2).
\end{align*}
\end{small}
\hspace{-0.12cm}Leting $K\rightarrow\infty$ in the above inequality and using the fact that $\sum_{k=0}^\infty (\|\mathbf{t}_x^{k}\|_2+\|\mathbf{t}_y^{k}\|_2)<\infty$ yields
\begin{align}
\sum_{k=1}^{\infty}&\big(\|\mathbf{ P}^{k+1}-\mathbf{ P}^{k}\|_F+\|\mathbf{ Q}^{k+1}-\mathbf{ Q}^k\|_F+\|\mathbf{\widehat X}^{k+1}-\mathbf{\widehat X}^k\|_F \nonumber
\\&+\|\mathbf{\widehat Y}^{k+1}-\mathbf{\widehat Y}^k\|_F\big)<\infty,
\end{align} 
which implies that $\sum_{k=1}^{\infty}\|\mathbf{\widehat U}^{k+1}-\mathbf{\widehat U}^{k}\|_F<\infty $.

In the meanwhile, if the sequence $(\|\mathbf{\widehat U}^{k+1}-\mathbf{\widehat U}^{k}\|_F)$ is non-increasing,
then we have 
\begin{align}
2k\|\mathbf{\widehat U}^{2k+1}-\mathbf{\widehat U}^{2k}\|_F\leq 2\sum_{i=k+1}^{2k}\|\mathbf{\widehat U}^{i+1}-\mathbf{\widehat U}^{i}\|_F\rightarrow 0\nonumber
\end{align} 
as $k\rightarrow\infty$. Thus, we have $\|\mathbf{\widehat U}^{k+1}-\mathbf{\widehat U}^{k}\|_F=o(1/k)$, which in conjunction with an argument similar to Proposition 2 in~\cite{wang2015global}, implies that  the best running convergence rate of Algorithm~\ref{alg:ours} is $o(1/k)$.
\end{proof}
 
\subsection{Proof of Proposition~\ref{cl1}}\label{apen:E}
For each $i$, there exists an integer $s_i\leq k_x$ such that $i\in\mathcal{G}_{s_i}^x$. Using the definition~(\ref{mathX}), we have $\|\mathbf{x}_j-\mathbf{x}_i\|_2<2C_x/K$ for any $j\in\mathcal{G}_{s_i}^x$ and $\|\mathbf{x}_j-\mathbf{x}_i\|_2\geq 2C_x/K$ for any $j\in(\mathcal{G}^x_{s_i})^c=\bigcup_{1\leq t\leq k_x,\, t\neq s_i} \mathcal{G}_{t}^x$.  Recalling the definition of $K_i$, we have $|(\mathcal{G}^x_{s_i})^c |=n-K_i$.
These facts imply 
\begin{align}%\label{kiki}
\sum_{j=1}^{n} \frac{\Upsilon(\mathbf{x}_i-\mathbf{x}_j)}{K_i(n-K_i)}&=\left(\sum_{j\in\mathcal{G}_{s_i}^x} +\sum_{j\in(\mathcal{G}_{s_i}^x)^c} \right)\frac{\Upsilon(\mathbf{x}_i-\mathbf{x}_j)}{K_i(n-K_i)}\nonumber
\\&=\sum_{j\in(\mathcal{G}_{s_i}^x)^c} {(K_i(n-K_i))^{-1}}= K_i^{-1}, \nonumber
\end{align}
which implies
\begin{align}\label{kkj}
\sum_{i,j} w_{ij}\Upsilon(\mathbf{x}_i-\mathbf{x}_j)=  \sum_{i=1}^n K_i^{-1}=k_x =|\mathcal{G}(\mathbf{X})|
\end{align}
where the last equality follows from the fact that $K_i=K_j$ for any $i,j\in\mathcal{G}_s^x,1\leq s\leq k_x$.  

\subsection{Proof of Theorem~\ref{th3}}\label{apen:F}
Let $p(x)$ and $q(x)$ denote probability density functions (pdf) of a random variable. Recall the Kullback-Leibler divergence $\textup{D}(p\|q)\overset{\Delta}{=}\mathbb{E}_p\big[\log (p/q)\big]$ and the Hellinger affinity $\textup{A}(p,q)\overset{\Delta}{=}\mathbb{E}_p\big[ \sqrt {q/p}\big ]$.
We first quote the following lemma from~\cite[Lemma A.1]{soni2016noisy}.
\begin{lemma}\label{le:a1}
Let $\mathbf{M}^*$ be an $n\times m$ target matrix and let $\mathcal{H}$ be a finite collection of candidate estimates $\mathbf{H}$ of $\,\mathbf{M}^*$, each of which has a penalty $\textup{pen}(\mathbf{H})\geq 1$ such that $\sum_{\mathbf{H}\in\mathcal{H}}2^{-\textup{pen}(\mathbf{H})}\leq 1$. 
 Assume each $(i,j)\in[n]\times[m]$ is included in the observed index set $\Omega$ with probability $\gamma=s/(nm)$ for certain integer $s<nm$ and  the joint pdf of the observations $\Phi_{\Omega}(\mathbf{M})$ follows $ p_{\mathbf{M}^*}(\Phi_{\Omega}(\mathbf{M}))=\prod_{(i,j)\in\Omega}\,p_{M^*_{ij}}(M_{ij})$, which are  independent conditioned on $\Omega$. 
Given a constant $C_D$ satisfying  
\begin{align}\label{cdcd1}
C_D\geq\max_{\mathbf{H}\in\mathcal{H}}\;\max_{(i,j)\in[n]\times[m]}\textup{D}(p_{M^*_{ij}}(M_{ij})\|p_{H_{ij}}(M_{ij})),
\end{align} 
we have that for any $\xi\geq \left(1+2C_D/3\right)2\log 2$, the complexity regularized maximum likelihood estimator 
\begin{align}
\mathbf{\widehat M}(\Omega,\Phi_{\Omega}(\mathbf{M}))=\arg\min_{\mathbf{H}\in\mathcal{H}}\left\{ -\log  p_{\mathbf{H}}(\Phi_{\Omega}(\mathbf{M})) +\xi\textup{pen}(\mathbf{H})\right\} \nonumber
\end{align}
satisfies the estimation error bound
\begin{align}
\mathbb{E}&\left[-2\log\textup{A}(p_{\mathbf{\widehat M}}(\mathbf{M}),p_{\mathbf{M^*}}(\mathbf{M}) )\right]/(nm)  \nonumber
\\\leq& 8C_D\log s/s+3\min_{\mathbf{H}\in\mathcal{H}}\big\{ \textup{D}(p_{\mathbf{M^*}}(\mathbf{M})\|p_{\mathbf{H}}(\mathbf{M}))/(nm)\nonumber
\\&+\left(\xi+4C_D\log 2/3\right)\textup{pen}(\mathbf{H})/s\big\},
\end{align} 
where the expectation is with respect to the joint distribution of $\,\left(\Omega,\Phi_{\Omega}(\mathbf{M})\right)$ and the shorthands 
\begin{small}
\begin{align}
\textup{A}(p_{\mathbf{\widehat M}}(\mathbf{M}),p_{\mathbf{M^*}}(\mathbf{M}))&=\prod_{i,j} \textup{A}(p_{\widehat M_{ij}}(M_{ij}),p_{M^*_{ij}}(M_{ij})),\nonumber
\\ \textup{D}(p_{\mathbf{M^*}}(\mathbf{M})\|p_{\mathbf{H}}(\mathbf{M}))&=\sum_{i,j}\textup{D}(p_{M_{ij}^*}(M_{ij})\|p_{H_{ij}}(M_{ij})).
\end{align}
\end{small}
%\end{small}
\end{lemma}

In order to apply this lemma, we need to construct penalties $\textup{pen}(\mathbf{H})$ for each $\mathbf{H}\in\mathcal{H}$ defined in~(\ref{mathM}) such that $\sum_{\mathbf{H}\in\mathcal{H}}2^{-\textup{pen}(\mathbf{H})}\leq 1$.
Let $\mathcal{H}_1\overset{\Delta}{=}\big\{\mathbf{H=XY}^T:\mathbf{X}\in\mathcal{X},\mathbf{Y}\in\mathcal{Y}\big\}$ with $\mathcal{X}$ and $\mathcal{Y}$ defined in~(\ref{mathX}). Based on~\cite{cover2012elements}, the inequality 
$\sum_{\mathbf{H}\in\mathcal{H}_1}2^{-\textup{pen}(\mathbf{H})}$$\leq 1$ can be satisfied  by choosing $\textup{pen}(\mathbf{H})$ to be the code length of some uniquely decodable binary code. Thus, we use the following three steps to design $\textup{pen}(\mathbf{H})$. Recall $K=2^{\lceil \mu\log_2( n \vee m )\rceil}$ and let $n_0= n \vee m $. 
\begin{enumerate} 
\item First encode the number of subgroups of $\{\mathbf{x}_i,1\leq i\leq n\}$, i.e. $|\mathcal{G}(\mathbf{X})|$, with $\lceil\log_2(n)\rceil$ bits. Similarly, encode $|\mathcal\mathcal{G}(\mathbf{Y})|$ with $\lceil\log_2(m)\rceil$ bits. 
\item For each $\mathbf{x}_i,1\leq i\leq n$, encode the index of the subgroup that $\mathbf{x}_i$ belongs to with $\lceil\log_2(n)\rceil$ bits. Similarly, for each $\mathbf{y}_j,1\leq j\leq m$, encode the index of the subgroup that $\mathbf{y}_j$ belongs to with $\lceil\log_2(m)\rceil$ bits. Thus, the total number of bits used in this step is $n\lceil\log_2(n)\rceil+m\lceil\log_2(m)\rceil$.
\item Finally, using~(\ref{mathX}), we encode $X_{ij}$ with $\log_2(K)$ bits. Similarly, encode $Y_{ij}$ with $\log_2(K)$ bits. Since all row vectors in a subgroup are identical, the total number of bits used in this step is $d\log_2(K)(|\mathcal{G}(\mathbf{X})|+|\mathcal{G}(\mathbf{Y})|)$.
\item We assign each $\mathbf{H}\in\mathcal{H}_1$  with a code that is concatenation of the code for $\mathbf{X}$ followed by  the code for $\mathbf{Y}$. Thus, the penalty 
\begin{align}\label{penM}
\textup{pen}(\mathbf{H})=&(n+1)\lceil\log_2(n)\rceil+(m+1)\lceil\log_2(m)\rceil\nonumber
\\&+d\log_2(K)(|\mathcal{G}(\mathbf{X})|+|\mathcal{G}(\mathbf{Y})|). 
\end{align}
Since such codes are uniquely decodable, we can obtain that  $\sum_{\mathbf{H}\in\mathcal{H}_1}2^{-\textup{pen}(\mathbf{H})}\leq 1$.
\end{enumerate}
Since $\mathcal{H}\subseteq\mathcal{H}_1$,   $\sum_{\mathbf{H}\in\mathcal{H}}$$2^{-\textup{pen}(\mathbf{H})}\leq\sum_{\mathbf{H}\in\mathcal{H}_1}2^{-\textup{pen}(\mathbf{H})}\leq 1$.
Then, replacing $\textup{pen}(\mathbf{H})$ in Lemma \ref{le:a1} by~(\ref{penM}) implies 
\begin{align}\label{mhat}
\mathbf{\widehat M}=\arg\min_{\mathbf{H}\in\mathcal{H}}&\Big\{ -\log  p_{\mathbf{H}}(\Phi_{\Omega}(\mathbf{M}))\nonumber
  \\& +\xi d\log_2(K)(|\mathcal{G}(\mathbf{X})|+|\mathcal{G}(\mathbf{Y})|)\Big\},
\end{align}
which, based on the fact that  $\max\{\lceil\log_2(n)\rceil,\lceil\log_2(m)\rceil\}\leq\log_2(K)=\lceil\mu\log_2(n_0)\rceil$, satisfies
\small
\begin{align}\label{ehat}
\mathbb{E}&\left[-2\log\textup{A}(p_{\mathbf{\widehat M}}(\mathbf{M}),p_{\mathbf{M^*}}(\mathbf{M}) )\right]/(nm)\nonumber
\\ \leq &\frac{8C_D\log s}{s}+3\min_{\mathbf{H}\in\mathcal{H}}\bigg\{ \frac{\textup{D}(p_{\mathbf{M}^*}(\mathbf{M})\|p_{\mathbf{H}}(\mathbf{M}))}{nm}+\left(\xi+\frac{4C_D\log 2}{3}\right)\nonumber
\\&d\log_2(K)\times\frac{(m+n+2)d^{-1}+|\mathcal{G}(\mathbf{X})|+|\mathcal{G}(\mathbf{Y})|}{s}\bigg\}.
\end{align}
\normalsize
 \hspace{-0.13cm}Recall that $\xi\geq \left(1+2C_D/3\right)2\log 2$.  Using the fact that  $\log_2(K)=\lceil\mu \log_2 (n_0)\rceil\leq 2\mu\log(n_0)/\log 2$ and letting $\lambda=2\sigma^2\xi d\log_2(K)$, for any
\begin{align}\label{lamlam}
\lambda\geq 8\sigma^2\mu d  \left(1+2C_D/3\right) \log (n_0), 
\end{align}
the estimate of
\begin{align}\label{emiemi}
\mathbf{\widehat M}(\Omega,\Phi_{\Omega}(\mathbf{M}))=\arg&\min_{\mathbf{H}\in\mathcal{H}}\Big\{ -\log  p_{\mathbf{H}}(\Phi_{\Omega}(\mathbf{M})) \nonumber
\\&+\frac{\lambda}{2\sigma^2}(|\mathcal{G}(\mathbf{X})|+|\mathcal{G}(\mathbf{Y})|)\Big\}
\end{align}
satisfies 
\begin{small}
\begin{align}\label{eeel}
&\mathbb{E}\left[-2\log\textup{A}(p_{\widehat M}(\mathbf{M}),p_{\mathbf{M^*}}(\mathbf{M}) )\right]/(nm) \nonumber
\\&\leq 8C_D\frac{\log s}{s}+3\min_{\mathbf{H}\in\mathcal{H}}\Big\{ \textup{D}(p_{\mathbf{M}^*}(\mathbf{M})\|p_{\mathbf{H}}(\mathbf{M}))/(nm)+\bigg(\frac{\lambda}{2\sigma^2}+ \nonumber
\\&\frac{8\mu d C_D\log(n_0)}{3}\bigg)\times\frac{(m+n+2)/d+|\mathcal{G}(\mathbf{X})|+|\mathcal{G}(\mathbf{Y})|}{s}\Big\}.
\end{align}
\end{small}
Recall that the noise matrix $\mathbf{N}$ contains i.i.d. $\mathcal{\mathcal{N}}(0,\sigma^2)$ elements and the pdf of the observations $\Phi_{\Omega}(\mathbf{M})$ is
\begin{align*}
p_{\mathbf{H}}(\Phi_{\Omega}(\mathbf{M}))=\frac{1}{(2\pi\sigma^2)^{|\Omega|/2}}\exp\left(-\frac{\|\Phi_{\Omega}(\mathbf{M-H})\|^2_F}{2\sigma^2}\right),
\end{align*} 
we have \eqref{esjky} is of the same form as \eqref{emiemi}. Furthermore, we have
%\begin{small}
\begin{align}\label{dmp}
\textup{D}&(p_{M^*_{ij}}(M_{ij})\|p_{H_{ij}}(M_{ij})) \nonumber
\\=&\int_{-\infty}^{\infty}\frac{H^2_{ij}-(M^*_{ij})^2+2(M^*_{ij}-H_{ij})x}{2\sigma^2} p_{M^*_{ij}}(x) dx\nonumber
\\=&\frac{1}{2\sigma^2}\left(H^2_{ij}-(M^*_{ij})^2+2(M^*_{ij}-H_{ij})\int_{-\infty}^{\infty}xp_{M^*_{ij}}(x)dx\right)\nonumber
\\\overset{(i)}=&2\sigma^{-2}\left(H_{ij}-M^*_{ij}\right)^2,
\end{align}
%\end{small}
where $(i)$ follows from the fact that $\int_{-\infty}^{\infty}xp_{M^*_{ij}}(x)dx=M_{ij}^*$. Recall~(\ref{cdcd1}) and note that $\|\mathbf{H}\|_{\infty}\leq C_m$.  Then picking $C_D=2C_m^2/\sigma^2$ in (\ref{lamlam}), we have   
\begin{align}\label{lamlam2}
\lambda\geq 8\mu d  \left(\sigma^2+4C_m^2/3\right) \log (n_0).
\end{align}
 Using a similar approach, we obtain
\begin{align}
-\log\textup{A}(p_{M^*_{ij}}(M_{ij}),p_{H_{ij}}(M_{ij}))=\left(H_{ij}-M^*_{ij}\right)^2/(8\sigma^2). \nonumber
\end{align}
which, in conjunction with~(\ref{dmp}) and (\ref{eeel}), implies
\begin{small}
\begin{align}
\mathbb{E}&\left[\|\mathbf{\widehat M-M^*}\|_F^2\right]/(nm)\nonumber
\\\leq& \frac{64C_m^2\log s}{s}+6\min_{\mathbf{H}\in\mathcal{H}}\bigg\{ \frac{\|\mathbf{H-M^*}\|_F^2}{nm}+\left(\lambda+\frac{32\mu d C_m^2\log(n_0)}{3}\right)\nonumber
\\&\times\frac{(m+n+2)d^{-1}+|\mathcal{G}(\mathbf{X})|+|\mathcal{G}(\mathbf{Y})|}{s}\bigg\}.
\end{align}
\end{small}
\hspace{-0.15cm}Next, we prove an upper bound for $\min_{\mathbf{H}\in\mathcal{H}}\| \mathbf{H-M^*}\|_F^2$. Let $\mathbf{H}_0=\mathbf{X}_0\mathbf{Y}^T_0\in\mathcal{M}$ be a candidate reconstruction such that the  entries of $\mathbf{X}_0$ and $\mathbf{Y}_0$ are the closest discretized surrogates  of the entries of $\mathbf{X}^*$ and $\mathbf{Y}^*$, respectively. Recalling~(\ref{mathM}) and~(\ref{mathX}), we have, $\|\mathbf{X}_0-\mathbf{X}^*\|_F^2\leq 4dnC^2_x/K^2$ and $\|\mathbf{Y}_0-\mathbf{Y}^*\|_F^2\leq 4dmC_y^2/K^2$, which implies
\begin{align}
\|\mathbf{H}_0-\mathbf{M}^*\|_F&=\|\mathbf{X}_0\mathbf{Y}^T_0-\mathbf{X}^*(\mathbf{Y}^*)^T\|_F\nonumber
\\&\leq \|\mathbf{X}_0-\mathbf{X}\|_F\|\mathbf{Y}_0\|_F+\|\mathbf{X}^*\|_F\|\mathbf{Y}_0-\mathbf{Y}^*\|_F\nonumber
\\&\leq \frac{4dC_xC_y\sqrt{nm}}{K}\leq \frac{4dC_xC_y\sqrt{nm}}{\mu\log(n_0)},  \nonumber
\end{align}
and hence
\begin{align}%\label{cxc2}
\|\mathbf{H}_0-\mathbf{M}^*\|_F^2/(nm)\leq 16dC^2_xC^2_y/(\mu\log(n_0))^2,\nonumber
\end{align}
which, in conjunction with the fact that $|\mathcal{G}(\mathbf{X}_0)|+|\mathcal{G}(\mathbf{Y}_0)|=|\mathcal{G}(\mathbf{X}^*)|+|\mathcal{G}(\mathbf{Y}^*)|$, implies
\begin{align}\label{fbound}
\mathbb{E}&\left[\|\mathbf{\widehat M}-\mathbf{M}^*\|_F^2\right]/(nm) \nonumber
\\&\leq \frac{64C_m^2\log s}{s}+ \frac{96dC^2_xC^2_y}{(\mu\log(n_0))^2}+\left(6\lambda+64\mu C_m^2\log(n_0)\right)\times\nonumber
\\&\frac{(m+n+2)+(|\mathcal{G}(\mathbf{X^*})|+|\mathcal{G}(\mathbf{Y^*})|)d}{s} .
\end{align}
Letting $C_1=64C_m^2$, $C_2=96C^2_xC^2_y$, $C_3=64  C_m^2$ and using~(\ref{lamlam2}),~(\ref{fbound}) finishes the proof.

\end{document}